%% file: Main.tex
\documentclass[lettersize,onecolumn]{IEEEtran}
\usepackage{amsmath,amsfonts}
\usepackage{algorithmic}
\usepackage{algorithm}
\usepackage{array}
\usepackage[caption=false,font=normalsize,labelfont=sf,textfont=sf]{subfig}
\usepackage{textcomp}
\usepackage{stfloats}
\usepackage{url}
\usepackage{verbatim}
\usepackage{graphicx}
\usepackage{cite}
\usepackage[utf8]{inputenc} 
\usepackage[T1]{fontenc}    
\usepackage{booktabs}       
\usepackage{amsfonts}       
\usepackage{nicefrac}       
\usepackage{amsmath} 
\usepackage{times}
\usepackage{subfig}
\usepackage{amsmath}
\usepackage{dsfont}
\usepackage{xr}
\usepackage{tikz}
\usepackage{pgfplots}
\usepackage{adjustbox}
\usepackage{bm, braket}
\usepackage{amsthm}
\usepackage{natbib}
\usepackage{doi}
\usepackage{dsfont}
\usepackage{mathtools} 
\usepackage{nicematrix}
\usepackage{extarrows} 

\input{symbols}

\newtheorem{theorem}{Theorem}
\newtheorem{remark}{Remark}
\newtheorem{lemma}{Lemma}
\newtheorem{corollary}{Corollary}

\DeclareMathOperator*{\argmax}{arg\,max}
\DeclareMathOperator*{\argmin}{arg\,min}

\begin{document}

\title{A Novel Gaussian Min-Max Theorem and its Applications}

\author{%
 \IEEEauthorblockN{Danil Akhtiamov$^*$, Reza Ghane$^*$, Nithin K Varma and Babak Hassibi}
 
\IEEEauthorblockA{California Institute of Technology\\
                    Pasadena, CA, USA\\
                    Email: \{dakhtiam, rghanekh, nkanumur, hassibi\}@caltech.edu}
                    
 \and
 
 \IEEEauthorblockN{David Bosch$^*$}
 
 \IEEEauthorblockA{	Chalmers University of Technology\\
	Gothenburg, Sweden \\
        Email: davidbos@chalmers.se}
}



\maketitle
\def\thefootnote{*}\footnotetext{Equal contribution}
\begin{abstract}%
A celebrated result by Gordon allows one to compare the min-max behavior of two Gaussian processes if certain inequality conditions are met. The consequences of this result include the Gaussian min-max (GMT) and convex Gaussian min-max (CGMT) theorems which have had far-reaching implications in high-dimensional statistics, machine learning, non-smooth optimization, and signal processing. Both theorems rely on a pair of Gaussian processes, first identified by Slepian, that satisfy Gordon's comparison inequalities. In this paper, we identify a new pair of Gaussian processes satisfying these inequalities. The resulting theorems extend the classical GMT and CGMT Theorems from the case where the underlying Gaussian matrix in the primary process has iid rows to where it has independent but non-identically-distributed ones. The new CGMT is applied to the problems of multi-source Gaussian regression, as well as to binary classification of general Gaussian mixture models.
\end{abstract}

\section{Introduction and Motivation}

Assume the matrix $G\in{\bbR}^{n\times d}$, the vectors $g\in{\bbR}^d$ and $h\in{\bbR}^n$, and the scalar $\gamma$ all have i.i.d ${\cal N}(0,1)$, i.e., standard normal entries. Then, considering deterministic vectors $w\in{\bbR}^d$ and $v\in{\bbR}^n$, belonging to some sets ${\cal S}_w$ and ${\cal S}_v$, we can obtain the following pair of Gaussian processes
\[ X_{wv} = v^TGw+\gamma\|w\|_2\|v\|_2~~~\mbox{and}~~~Y_{wv} = \|v\|_2g^Tw+\|w\|_2h^Tv,\]
which we will refer to as a {\em Slepian pair} \citep{Slepian1962}. For such Slepian pairs, it can be readily shown that
\[ \bbE X_{wv}X_{w'v'}- \bbE Y_{wv}Y_{w'v'} = \left(\|w\|_2\|w'\|_2-w^Tw')
(\|v\|_2\|v'\|_2-v^Tv'\right) \]

By recalling Cauchy-Schwarz, this implies that the processes $X_{wv}$ and $Y_{wv}$ satisfy the following (so-called) comparison inequalities
\begin{equation}
\bbE X_{wv}X_{w'v'}\leq \bbE Y_{wv}Y_{w'v'}~~~\mbox{and}~~~
    \bbE X_{wv}X_{wv'} \geq \bbE Y_{wv}Y_{wv'}
    \label{eq:comp_ineq}
\end{equation}

For Gaussian processes that satisfy the inequalities \eqref{eq:comp_ineq}, \cite{Gordon1985SomeIF} showed that it is possible to relate 
\[ \bbP\left(\min_{w\in{\cal S}_w}\max_{v\in{\cal S}_v}X_{wv}\geq \psi(w,v)\right) \] to \[\bbP\left(\min_{w\in{\cal S}_w}\max_{v\in{\cal S}_v}Y_{wv}\geq \psi(w,v)\right), \]
for any function $\psi(w,v)$. We refer the reader to the next section for more details. This remarkable result allows one to infer properties of a so-called {\em primary optimization} (PO)
\[ \min_{w\in{\cal S}_w}\max_{v\in{\cal S}_v}v^TGw+\psi(w,v)
\]
in terms of a simpler so-called {\em auxiliary optimization} (AO) 
\[ \min_{w\in{\cal S}_w}\max_{v\in{\cal S}_v}
\|v\|_2g^Tw+\|w\|_2h^Tv+\psi(w,v). 
\]
This is the gist of the Gaussian min-max (GMT) \citep{Gordon88} and convex Gaussian min-max (CGMT) \citep{ThrampoulidisCGMT} theorems. The CGMT, in particular, has found numerous recent applications where, along with statistical physics-inspired methods, such as the replica method and approximate message passing (AMP), it has been used to derive sharp bounds on the performance of a variety of problems in high-dimensional statistics, machine learning, non-smooth convex optimization, and signal processing (e.g., \citep{stojnic2013framework, salehi2019impact, akhtiamov2023regularized, pmlr-v119-mignacco20a, 8365826, javanmard2022precise,montanari2019generalization,salehi2020performance,aolaritei2023, zhou2024uniform})


In this paper, we identify a novel pair of Gaussian processes that satisfy Gordon's inequalities (\ref{eq:comp_ineq}). This in turn allows for comparison between a new primary and auxiliary optimization, which, as we show below, can express problems which the known CGMT can not analyze.

The new pair of processes are as follows:
\[ X_{w,v_1,\ldots,v_k} = \sum_{\ell = 1}^k \left[v_{\ell}^T G_\ell\Sigma_{\ell}^{\frac{1}{2}} w + \gamma_\ell
\|\Sigma_{\ell}^{\frac{1}{2}} w \|_2\|v_\ell\|_2\right],
\]
and
\[ Y_{w,v_1,\ldots,v_k} = \sum_{\ell = 1}^k \left[\|v_{\ell}\|_2
g_\ell^T\Sigma_{\ell}^{\frac{1}{2}} w + \|\Sigma_{\ell}^{\frac{1}{2}} w \|_2h_\ell^Tv_\ell\right],
\]

where $\{G_\ell\in{\mathbb{R}}^{n_\ell\times d}\}_{\ell = 1}^k$, 
$\{g_\ell\in{\mathbb{R}}^d\}_{\ell = 1}^k$, 
$\{h_\ell\in{\mathbb{R}}^{n_\ell}\}_{\ell = 1}^k$ and $\{\gamma_\ell \in {\mathbb{R}}\}_{\ell=1}^k$
all have i.i.d ${\cal N}(0,1)$ standard normal entries and $\Sigma_{\ell}^{1/2}$ are $d\times d$ are positive semi-definite (PSD) matrices.

The remainder of the paper is devoted to obtaining the CGMT counterpart of the above pair of Gaussian random processes and to applying this novel CGMT to analyze two specific problems: firstly, multi-source Gaussian regression and, secondarily, binary classification of general Gaussian mixture models. We note that neither problem lends itself to analysis via the conventional CGMT.

\section{Related works}

We would like to highlight two recent works that provide other interesting examples of Gaussian Processes satisfying \eqref{eq:comp_ineq}. \cite{dhifallah2021inherent} introduce another generalization of the Convex Gaussian Min-Max Theorem that they use to analyze the effects of noise injection on the generalization performance of Random Feature Models.  \cite{jain2024scaling} apply Gaussian Comparison Inequalities to analyze a pair of processes that arises in their analysis of the performance of linear models trained on a mixture of true and synthetic data. An interested reader can see Remarks \ref{rem: montanari} and \ref{rem: yue_lu} in Section \ref{sec: main} for more details. 

\section{Preliminaries}

\subsection{Gordon's comparison inequality for Gaussian Processes} 

Introduced by Gordon in his landmark paper \citep{Gordon1985SomeIF}, the following comparison inequality is instrumental in proving the main results of our paper:   

\begin{theorem}[Gordon's comparison inequality \cite{Gordon1985SomeIF}]\label{thm: comparison} Let $I$ and $J$ be two \emph{finite} sets and $\{X_{ij}\}_{i \in I, j \in J}$, $\{Y_{ij}\}_{i \in I, j \in J}$ be two real valued centered Gaussian processes, which satisfy the following conditions:
    \begin{enumerate}
        \item $\bbE X_{ij}^2 = \bbE Y_{ij}^2$ for all $i \in I, j \in J$
        \item $\bbE X_{ij}X_{ij'} \ge \bbE Y_{ij}Y_{ij'}$ for all $i \in I, j, j'\in J$
        \item $\bbE X_{ij}X_{i'j'} \le \bbE Y_{ij}Y_{i'j'}$ for all $i \ne i' \in I, j, j'\in J$
    \end{enumerate}
    Let $\{t_{ij}\}_{i \in I, j \in J}$ be an arbitrary sequence of real numbers indexed by $(i, j) \in I \times J$. Then the following inequality holds:
    \begin{align*}
         \bbP\left(\bigcap_{i\in I} \bigcup_{j \in J} [ Y_{ij} \ge t_{ij} ]\right) \ge \bbP\left(\bigcap_{i\in I} \bigcup_{j \in J} [ X_{ij} \ge t_{ij} ]\right)
    \end{align*}
\end{theorem}

Now, let $G\in \bbR^{n\times d}$, $\gamma\in \bbR$, $g \in \bbR^d$, $h \in \bbR^n$ all have i.i.d $\calN(0,1)$ components, $I$ and $J$ be finite sets and $\{w_i\}_{i \in I} \subset \mathbb{R}^d$, $\{u_j\}_{j \in J} \subset \mathbb{R}^n$ be arbitrary discrete subsets. It was first observed in \cite{Gordon88} that the following processes satisfy the assumptions of Theorem \ref{thm: comparison}:
\begin{align*}
    & Y_{ij} = v_j^TGw_i  + \gamma \|w_i\|_2 \|v_j\|_2\\
    & X_{ij} = \|v_j\|_2 g^T w_i + \|w_i\|_2 h^T v_j
\end{align*}


By making an appropriate choice for the values of $\{t_{ij}\}_{i \in I, j \in J}$, and by proving that finite discrete subsets can be extended to compact sets, the following theorem can be proven. This theorem provides a probabilistic upper bound on the value of a more complicated optimization problem $\Phi$, which extends $Y_{ij}$, being larger than some threshold $t$, by means of the same quantity for an alternative optimization $\phi$, which extends $X_{ij}$. The alternative problem, $\phi$, is in general easier to analyze.

\begin{theorem}
\label{thm: Gordon GMT}
     Let $\calS_w \subset \mathbb{R}^d, \calS_v  \subset \mathbb{R}^n$ be compact sets and $\psi(w,v): \calS_w \times \calS_v \to \bbR$ be a continuous function.  Let $G\in \bbR^{n\times d}$, $\gamma\in \bbR$, $g \in \bbR^d$, $h \in \bbR^n$ all have i.i.d $\calN(0,1)$ components, and be independent of each other. Then letting
    \begin{align*}
        \Phi (G, \gamma) := \min_{w\in \calS_w} \max_{v\in \calS_v} v^T G w + \gamma \|w\|_2 \|v\|_2 + \psi(w,v) \\
        \phi (g, h) := \min_{w\in \calS_w} \max_{v\in \calS_v} \|v\|_2 g^T w + \|w\|_2 h^T v + \psi(w,v)
    \end{align*}
    We have for any $t \in \bbR$
    \begin{align*}
        \bbP( \Phi (G, \gamma) < t ) \le \bbP( \phi (g, h) < t)
    \end{align*}
\end{theorem}

The theorem above has an immediate corollary:

\begin{corollary}[Gaussian Min-Max Theorem]\label{cor: GMT}
    Assume the same setup as Theorem \ref{thm: Gordon GMT}. Then letting
    \begin{align*}
        & \Phi (G) := \min_{w\in \calS_w} \max_{v\in \calS_v} v^T G w + \psi(w,v) \\
        & \phi (g, h) := \min_{w\in \calS_w} \max_{v\in \calS_v} \|v\|_2 g^T w + \|w\|_2 h^T v + \psi(w,v)
    \end{align*}
    we have for any $t \in \bbR$
    \begin{align*}
        \bbP( \Phi (G, \gamma) < t ) \le 2\bbP( \phi (g, h) < t)
    \end{align*}
\end{corollary}

\subsection{Standard CGMT} 

Informally, the corollary above has the following interpretation: if $\phi(g,h)$ concentrates to a value $c$, then $\Phi (G, \gamma) < c + \epsilon$ holds with high probability for $\epsilon > 0$. 
While the above corollary is sufficient for some applications, it has two distinct drawbacks. Firstly, it only provides a probabilistic upper bound, without a corresponding lower bound the exact value of $\Phi$ cannot be determined asymptotically. Secondarily, it only provides information about the value of $\Phi$ in terms of the optimal value of $\phi$ but provides no information about the relationship between their solutions.

It turns out that the inequality from the Corollary \ref{cor: GMT} can be reversed under the additional assumptions that $\psi(w,v)$ is convex-concave and that the sets ${\cal S}_w$ and ${\cal S}_v$ are convex. Furthermore, in this case, if the solution of $\phi$ belongs to an open set $\calS \subset \calS_w$ with high probability then so does the solution of $\Phi$.  Formally, the following theorem holds (see \citep{ThrampoulidisCGMT, thrampoulidis2015regularized,  thrampoulidis2016recovering}):

\begin{theorem}[Convex Gaussian Min-Max Theorem]\label{thm: cgmt}  
    Let $\calS_w \subset \mathbb{R}^d, \calS_v  \subset \mathbb{R}^d$ be compact convex sets and $\psi(w,v): \calS_w \times \calS_v \to \bbR$ be convex-concave.  Furthermore, let $G\in \bbR^{n\times d}$, $g \in \bbR^d$, $h \in \bbR^n$ all have i.i.d $\calN(0,1)$ components. Define
    \begin{align*}
        & \Phi (G) := \min_{w\in \calS_w} \max_{v\in \calS_v} v^T G w  + \psi(w,v) \\
        & \phi (g, h) := \min_{w\in \calS_w} \max_{v\in \calS_v} \|v\|_2 g^T w + \|w\|_2 h^T v + \psi(w,v)
    \end{align*}
    The following statements hold:
    \begin{enumerate}
        \item For any $\eta, c \in \bbR$ 
        \begin{align*}
            \bbP( |\Phi (G) - c | > \eta ) \le 2\bbP( |\phi (g, h) - c| > \eta)
        \end{align*}
        \item Let $\calS$ be an arbitrary open subset of $\calS_w$ and $\calS^c = \calS_w \setminus \calS $. Define:
        \begin{align*}
        & \Phi_{\calS^c}(G) := \min_{w\in \calS^c} \max_{v\in \calS_v} v^T G w  + \psi(w,v) \\
        & \phi_{\calS^c}(g,h) := \min_{w\in \calS^c} \max_{v\in \calS_v} \|v\|_2 g^T w + \|w\|_2 h^T v + \psi(w,v) \\
        & w_\Phi (G) := \argmin_{w\in \calS_w} \max_{v\in \calS_v} v^T G w  + \psi(w,v)
    \end{align*}
     Assume also that there exist $\epsilon>0$, $\delta > 0$, $\barphi, \barphi_{\calS^c}$ such that
        \begin{itemize}
            \item $\barphi_{\calS^c} \ge \barphi + 3 \delta$
            \item $\phi(g,h) < \barphi + \delta$ with probability at least $1-\epsilon$
            \item $\phi_{\calS^c} > \barphi_{\calS^c} - \delta$ with probability at least $1-\epsilon$
        \end{itemize}
        Then
        \begin{align*}
            \bbP(w_\Phi (G) \in \calS) \ge 1-4\epsilon
        \end{align*}
    \end{enumerate}
    
\end{theorem}

\subsection{Multi-source Gaussian regression}\label{sec:prelim:subsec:multisourceGaussian}
To demonstrate the power of the novel CGMT presented in this work, we will focus on two examples. Here we consider Multi-source Gaussian Regression. Consider a setup of $k$-measurement devices and consider a dataset of $\left\lbrace(x_{i}^l, y_{i}^l)\in\bbR^{d}\times\bbR\right\rbrace_{i=1, l=1}^{n, k}$ of $n$ $d-$dimensional observations by each of the $k$ devices, with corresponding labels $y_{i}^l$. We will define $y_l\in\bbR^n$ element-wise by $(y_l)_i = y_{i, l}$ and similarly define $X_l \in \bbR^{n\times d}$ by means of $(X_l)_{i, j} = (x^l_i)_j$. Our regression problem then takes the form of
\begin{equation}\label{eq:GaussianRegression}
    \min_{\theta\in\bbR^{d}} \frac{1}{nk}\sum_{l=1}^k \ell_l\left(y_l - \frac{1}{\sqrt{d}}X_{l}\theta \right) + R(\theta),
\end{equation}
where $\theta\in\bbR^{d}$ and $\ell_l$ and $R$ are convex loss and regularization functions respectively. We shall assume that $X_l$ are independent Gaussian matrices, defined by $X_l = G_l\Sigma_l^{1/2}$, where $G_l\in\bbR^{n\times d}$ have i.i.d standard Gaussian elements, and $\Sigma_l\in\bbR^{d\times d}$ are covariance matrices. Finally, we will assume the following model, which relates the labels to the datapoints:
\begin{equation*}
    y_l = \frac{1}{\sqrt{d}}X_l\theta^* + \nu_l,
\end{equation*}
where $\theta^*\in\bbR^{d}$ denotes the "true" relationship, and $\nu_l \in \bbR^n$ is i.i.d zero mean noise, with variance $\sigma_{\nu, l}^2$ independent of $X_l$. In other words, we assume some true linear relationship, between data and labels subject to corruption due to noise. We allow for different noise power for each channel.
We will specifically concentrate on the generalization error of this regression problem, which we will define by:
\begin{equation}\label{eq:multisource_gen_error}
    \calE_{gen} = \frac{1}{2k}\bbE\sum_{l=1}^k\norm{y_{l, new}  - \frac{1}{\sqrt{d}}x_{l, new}^T\hat{\theta}}^2
\end{equation}
where $x_{l, new}$ are drawn from the same distribution as $X_l$ and $y_{l, new} = \frac{1}{\sqrt{d}}x_{l, new}^T\theta^* + \nu_{l, new}$, where $\nu_{l, new}$ is new noise drawn from the same noise distribution. We discuss and analyze this multi-source regression problem in section \ref{sec:Applications:subsec:MultiSourceRegression}.

One cannot leverage the usual CGMT to tackle the problem in (\ref{eq:GaussianRegression}) which will be made explicit in the following. By setting $z_l := y_l - \frac{1}{\sqrt{d}}X_{l}\theta$ and introducing Lagrange multipliers:
\begin{align} \label{op: expandedPOreg}
      \min_{\theta, z_l \in\bbR^{d}} \max_{v_l}  R(\theta)  + \frac{1}{nk}\sum_{l=1}^k \left[\frac{1}{\sqrt{d}} v_l^T X_{l}\theta - v_l^T y_l + v_l^T z_l + \ell_l(z_l)\right] 
\end{align}
Defining $X^T= \begin{pmatrix}
    X_1^T | X_2^T |\hdots | X_k^T
\end{pmatrix}$ and \\ $v^T = \begin{pmatrix}
    v_1^T & v_2^T & \hdots  v_k^T
\end{pmatrix}$ allows us to rewrite \eqref{op: expandedPOreg} as
\begin{align}\label{op: expandedPOregsing}
    & \min_{\theta, z_l \in\bbR^{d}} \max_{v}  R(\theta) + \frac{1}{n k \sqrt{d}} v ^T X \theta + \frac{1}{nk} v^T z - \frac{1}{nk} v^T y  + \frac{1}{nk} \ell(z),
\end{align}
where $z^T = \begin{pmatrix}
    z_1^T \hdots z_k^T
\end{pmatrix} $, $y^T = \begin{pmatrix}
    y_1^T \hdots y_k^T
\end{pmatrix}$, \\ and $\ell(z) = \sum_{l=1}^k\ell_l(z_l)$. Now note that we cannot apply the standard CGMT to either of (\ref{op: expandedPOreg}) or (\ref{op: expandedPOregsing}) since in (\ref{op: expandedPOreg}) the vector $\theta$ couples the optimizations over $v_l, z_l$, and in (\ref{op: expandedPOregsing}), as discussed, the rows of the Gaussian matrix $X$ are not identically distributed. We therefore need a new CGMT that can apply to the Gaussian process, $v ^T X \theta = \sum_{l=1}^kv_lG_l\Sigma_l^{1/2}\theta$, which is the main contribution of the current paper.

\subsection{Binary classification for Gaussian Mixture Models} 
\label{subsec:bin_class}

We next consider the binary linear classification problem on data generated from a Gaussian Mixture Model (GMM) with arbitrary covariance. To be precise, the data $\mathbf{x}_i \in \mathbb{R}^d$ is generated as:
$$
\mathbf{x}_i \sim \mathcal{N}(\mu_k,\Sigma_k) \quad \text{where}\quad k \sim \textit{Unif}\{1,2\},
$$
$\mu_k \in \mathbb{R}^d$ are the mean vectors and $\Sigma_k \in \mathbb{R}^{d\times d}$ are the corresponding covariance matrices. For simplicity, we assume that we uniformly sample from both the mixtures but the analysis can be extended to arbitrary sampling priors. We set the labels $y_i=1$ if data is generated from the cluster with mean $\mu_1$ and $y_i=-1$ if data is generated from the cluster with mean $\mu_2$. We then formulate the learning objective in terms of the following regression problem.
\begin{equation}
\label{eq:classification}
\min_w \calL((A+M)w-z) + \lambda f(w)
\end{equation}
where $\lambda \in \mathbb{R}_{+}$ controls the degree of regularization, $M^T = 
    ( \underbrace{\begin{matrix}
       \mu_1 & ... &
    \mu_1
    \end{matrix}}_{\frac{n}{2}} \underbrace{\begin{matrix}
       \mu_2 & ... & 
    \mu_2 
    \end{matrix}}_{\frac{n}{2}})$ encodes the means corresponding to each class, $z^T = ( \underbrace{\begin{matrix}
       1 & ... &
    1
    \end{matrix}}_{\frac{n}{2}} \underbrace{\begin{matrix}
       -1 & ... & 
    -1 
    \end{matrix}}_{\frac{n}{2}})$
corresponds to the labels and 
$$
A = \begin{pmatrix}
    G_1 & 0  \\
    0 & G_2 
\end{pmatrix} \begin{pmatrix}
    \Sigma_1^{\frac{1}{2}} \\
    \Sigma_2^{\frac{1}{2}} \\
\end{pmatrix}
$$
where $G_1,G_2 \in \mathbb{R}^{\frac{n}{2}\times d}$ are Gaussian matrices with iid $\mathcal{N}(0,1)$ entries. Therefore, the rows of $A+M$ denote the data generated by the GMM and without loss of generality we assume that the data is ordered and split into two equal halves from the two mixtures. Additionally, we assume both loss function $\mathcal{L}(\cdot)$ and regularization function $f(\cdot)$ to be separable convex functions. Given a linear $\hat w$, the classification rule at inference is $\hat y(x) = \text{sign}(\hat w^Tx)$. The classification error, or the error on unseen data, generated for the above GMM is given as:
\begin{equation}
   E(\hat w) = \frac{1}{2}Q\left(\frac{\mu_1^T\hat{w}}{\sqrt{\hat{w}^T \Sigma_1 \hat{w}}}\right) + \frac{1}{2}Q\left(-\frac{\mu_2^T\hat{w}}{ \sqrt{\hat{w}^T \Sigma_2 \hat{w}}}\right) 
   \label{eq: error}
\end{equation}
where $Q(\cdot)$ is the standard Gaussian CCDF.

Finally, to make the results as explicit as possible we assume that $\mu_1$ and $\mu_2$ are i.i.d. $\mathcal{N}(0,1)$ and their matching entries are correlated according to $\bbE \mu_{1i} \mu_{2i} = r$ for some $-1 \le r \le 1$. Under this assumption we have $\|\mu_1\| \approx \|\mu_2\| \approx \sqrt{d}$ and $\mu_1^T\mu_2 \approx rd$.\footnote{Theorem \ref{th: classification master} further below is true for arbitrary $\mu_1$ and $\mu_2$ and one can still use it to derive versions of Theorem \ref{th: classification l2} and Corollary \ref{cor: rk_1_noise} for any other given model for $\mu_1$ and $\mu_2$ but we prefer to follow with this specific assumption because it is meant to capture the simplest case of two "typical" $d$-dimensional vectors that are not necessarily orthogonal.} 

To illustrate why one cannot just readily apply the standard CGMT in this setting, we have to spell out \eqref{eq:classification} further. The standard way of dealing with \eqref{eq:classification} is introducing \\ $u:= (A+M)w-z$ and inserting a Lagrange multiplier, after which we would have:
\begin{equation*}
    \min_{w,u} \max_v v^TAw + v^TMw-v^Tz-v^T u + \calL(u)  + \lambda f(w)
\end{equation*}
The standard CGMT cannot be applied here because $A$ is not i.i.d. To try resolving it, one can denote $v^T = (v_1^T, v_2^T)$ for $v_1, v_2 \in \mathbb{R}^{\frac{n}{2}}$ and obtain 
\begin{align}
    & \min_{w,u} \max_{v_1,v_2} v_1^TG_1\Sigma_1^{\frac{1}{2}}w + v_2^TG_2\Sigma_2^{\frac{1}{2}}w  + v^TMw-v^Tz-v^T u + \calL(u)  + \lambda f(w)
    \label{eq: class_obj_spelled}
\end{align}
Note first of all that the objective above does not separate into a CGMT instance with $G_1$ and a CGMT instance with $G_2$ because $w$ "sees" both of these matrices. If $\Sigma_1 = \Sigma_2$, we can define $\tilde{w} = \Sigma_2^{\frac{1}{2}}w$, $(M\Sigma_2^{-\frac{1}{2}})^T = \left(\tilde{M}_1 | \tilde{M}_2\right)^T, \tilde{M}_1, \tilde{M}_2 \in \mathbb{R}^{d \times \frac{n}{2} }$, $\tilde{f}(x) = 
f(\Sigma_2^{-\frac{1}{2}}x)$ and arrive at: 
\begin{equation*}
    \min_{\tilde{w},u} \max_{v} v^T(G_1^T|G_2^T)^T\tilde{w} + v^T\tilde{M}^T\tilde{w}-v^Tz-v^T u + \calL(u)  + \lambda \tilde{f}(\tilde{w})
\end{equation*}
Hence, when $\Sigma_1 = \Sigma_2$, since $(G_1^T|G_2^T)$ is i.i.d. Gaussian we {\it can} analyze \eqref{eq: class_obj_spelled} using the standard CGMT. However, for the case of general $\Sigma_1$ and $\Sigma_2$ the same argument {\it does not apply} and it appears that one just has to develop a framework allowing for terms of the form $v_1^TG_1\Sigma_1^{\frac{1}{2}}w + v_2^TG_2\Sigma_2^{\frac{1}{2}}w$ to be present in the objective. This is essentially what Theorem \ref{thm: gcgmt} addresses. 
\section{Main results}\label{sec: main}
We will occasionally shorten $\|\cdot\|_2$ to $\|\cdot\|$  for brevity in the expressions below. 

\begin{lemma}\label{lm: comp_X_Y}
Let  $G_1 \in \bbR^{n_1\times d}, \dots, G_k \in \bbR^{n_k\times d}$, $\gamma_1, \dots, \gamma_k \in \bbR$, $g_1, \dots, g_k \in \bbR^d$, $h_1 \in \bbR^{n_1}, \dots,  h_k \in \bbR^{n_k}$ all have i.i.d $\calN(0,1)$ components, $\alpha_1, \dots, \alpha_k: \bbR^d \to \bbR^d$ and  $\beta_1: \bbR^{n_1} \to \bbR^{n_1}, \dots, \beta_k: \bbR^{n_k} \to \bbR^{n_k}$ be arbitrary functions, $I, J_1, \dots, J_k$ be finite sets and $\{w_i\}_{i \in I} \subset \mathbb{R}^d$, $\{u_{j_1}\}_{j_1 \in J_1} \subset \mathbb{R}^{n_1}, \dots, \{u_{j_k}\}_{j_k \in J_k} \subset \mathbb{R}^{n_k}$ be arbitrary discrete subsets. Then the following two Gaussian Processes defined on $I \times J$, where $J = J_1 \times \dots \times J_k$ and $j \in J$ denotes $(j_1,\dots,j_k) \in J_1 \times \dots \times J_k$, satisfy the assumptions of Theorem \ref{thm: comparison}:
\begin{align*}
    & Y_{ij} = \sum_{\ell = 1}^k \left[\beta_{\ell}(v_{j\ell})^T G_\ell \alpha_\ell(w_i)  + \gamma_\ell \|\alpha_\ell(w_i)\| \|\beta_{\ell}(v_{j\ell})\|\right] \\
    & X_{ij} = \sum_{\ell = 1}^k \left[\|\beta_{\ell}(v_{j\ell})\| \alpha_\ell(w_i)^Tg_\ell  + \|\alpha_\ell(w_i)\| h_\ell^T\beta_{\ell}(v_{j\ell})\right]
\end{align*}

Moreover, the processes defined by $\tilde{X}_{ji} = X_{ij}$ and $\tilde{Y}_{ji} = Y_{ij}$ on $J \times I$ satisfy these assumptions as well. 
\label{lm: process_with_sigmas}
\end{lemma}

Consider arbitrary convex functions $\alpha_1, \dots, \alpha_k: \bbR^d \to \bbR^d$ and  $\beta_1: \bbR^{n_1} \to \bbR^{n_1}, \dots, \beta_k: \bbR^{n_k} \to \bbR^{n_k}$, and define: 
    \begin{align}
        \label{eq: phis_general}
         \Phi (G) :=& \min_{w\in \calS_w} \max_{v \in \calS_{v}} \sum_{\ell = 1}^k \beta_{\ell}(v_{\ell})^T G_\ell \alpha_\ell(w)  + \psi(w,v) \\
         \phi (g, h) :=& \min_{w\in \calS_w} \max_{v\in \calS_v} \sum_{\ell = 1}^k \left[ \|\beta_{\ell}(v_{\ell})\|_2 g_{\ell}^T \alpha_\ell(w) + \|\alpha_\ell(w)\|_2 h_{\ell}^T v_\ell\right] + \psi(w,v) \nonumber
    \end{align}

After making several additional technical steps, Lemma \ref{lm: comp_X_Y} combined with Theorem \ref{thm: comparison} can be used to deduce the following generalization of Theorem \ref{thm: cgmt}:

\begin{theorem}[Generalized CGMT]\label{thm: gcgmt}
    Let $\calS_w \subset \mathbb{R}^d, \calS_{v_1} \subset \mathbb{R}^{n_1} , \dots, \calS_{v_k}  \subset \mathbb{R}^{n_k}$ be compact convex sets. Denote $\calS_v := \calS_{v_1} \times \dots \times \calS_{v_k}$, let $v \in \calS_v$ stand for $(v_1, \dots, v_k) \in \calS_{v_1} \times \dots \times \calS_{v_k}$ and  $\psi(w,v): \calS_w \times \calS_{v} \to \bbR$ be continuous, $\alpha_1, \dots, \alpha_k: \bbR^d \to \bbR^d$ and  $\beta_1: \bbR^{n_1} \to \bbR^{n_1}, \dots, \beta_k: \bbR^{n_k} \to \bbR^{n_k}$ be also arbitrary continuous functions. Furthermore, let $G_1 \in \bbR^{n_1\times d}, \dots, G_k \in \bbR^{n_k\times d}$, $g_1, \dots, g_k \in \bbR^d$, $h_1 \in \bbR^{n_1}, \dots, h_k \in \bbR^{n_k}$ all have i.i.d $\calN(0,1)$ components and $G = (G_1, \dots, G_k)$, $g = \left(g_1, \dots, g_k\right)$, $h = \left(h_1, \dots, h_k\right)$ be the corresponding $k$-tuples. Define  $\Phi(G)$ and $\phi (g, h)$ as in \eqref{eq: phis_general}. Then the following statements hold:
    \begin{enumerate}
        \item    For any $t \in \bbR$ we have 
        \begin{align*}
        \bbP( \Phi (G, \gamma) < t ) \le 2^k\bbP( \phi (g, h) < t)
         \end{align*}
        \item Assume, moreover, that $\psi$ is convex-concave and $\alpha_\ell$, $\beta_\ell$ are affine functions of their inputs, that is they take the form of  $\alpha_\ell(w) = A_\ell w+ a_\ell, \beta_\ell(u) = B_\ell u + b_\ell$ for some $A_\ell \in \mathbb{R}^{d \times d}, B_\ell \in \mathbb{R}^{n_\ell \times n_\ell}, a_\ell \in \bbR^{d}, b_\ell \in \bbR^{n_\ell}$. Then for all $\eta, c \in \bbR$ such that $\eta > 0$ the following statement holds: 
        $$\bbP( |\Phi (G) - c | > \eta ) \le 2^k\bbP( |\phi (g, h) - c| > \eta)$$
        \item Assume the same setting as in (2). Let $\calS$ be an arbitrary open subset of $\calS_w$ and $\calS^c = \calS_w \setminus \calS $. Define:
        \begin{align*}
        & \Phi_{\calS^c}(G) := \min_{w\in \calS^c} \max_{v \in \calS_{v}}  \sum_{\ell = 1}^k \beta_{\ell}(v_{\ell})^T G_\ell \alpha_\ell(w)  + \psi(w,v)\\
        & \phi_{\calS^c}(g,h) := \min_{w\in \calS^c} \max_{v\in \calS_v} \sum_{\ell = 1}^k \sum_{\ell = 1}^k \left[ \|\beta_{\ell}(v_{\ell})\|_2 g_{\ell}^T \alpha_\ell(w) + \|\alpha_\ell(w)\|_2 h_{\ell}^T v_\ell\right] + \psi(w,v) \\
        & w_\Phi (G) := \argmin_{w\in \calS_w} \max_{v \in \calS_{v}} \sum_{\ell = 1}^k \beta_{\ell}(v_{\ell})^T G_\ell \alpha_\ell(w)  + \psi(w,v)
    \end{align*}
     Assume also that there exist $\epsilon, \delta >0$, $\barphi, \barphi_{\calS^c}$ such that
        \begin{itemize}
            \item $\barphi_{\calS^c} \ge \barphi + 3 \delta$
            \item $\phi(g,h) < \barphi + \delta$ with probability at least $1-\epsilon$
            \item $\phi_{\calS^c} > \barphi_{\calS^c} - \delta$ with probability at least $1-\epsilon$
        \end{itemize}
        Then $\bbP(w_\Phi (G) \in \calS) \ge 1-2^{k+1}\epsilon$
    \end{enumerate}
\end{theorem}

Part 2 of Theorem \ref{thm: gcgmt} says that if $\phi(g,h)$ concentrates to a value $c$, then so does $\Phi(G)$, albeit with a possibly slightly worse tail bound. Moreover, it turns out that $\phi(g,h)$ concentrates to $\bbE \phi(g,h)$ exponentially fast, which we prove in Appendix B. As such, Theorem \ref{thm: gcgmt} implies that $\Phi(G)$ concentrates exponentially fast to the same value. Part $3$ of the theorem says that if the solution of $\phi(g,h)$ satisfies a certain property with probability very close to $1$, such as for example belonging to a ball of a fixed radius, then the same is true for the solution of $\Phi(G)$. 

\begin{remark}\label{rem: yue_lu}
We would also like to point out Theorem 3 from \cite{dhifallah2021inherent} that considers similar objectives, namely ones of the following form, where the entries of $G_\ell$ are again i.i.d. $\mathcal{N}(0,1)$:
\begin{equation}\label{eq: theirs}
    \Phi (G) := \min_{w\in \calS_w} \max_{v \in \calS_{v}} \sum_{\ell = 1}^k v_{\ell}^T G_\ell w_{\ell}  + \psi(w,v)
\end{equation}

After establishing the corresponding comparison inequality, \cite{dhifallah2021inherent} proceed to apply it to study the effects of noise injection on the generalization performance of Random Feature Models.
Our work is concerned with characterizing the generalization performance for binary classification with arbitrary covariance matrices and multi-source Gaussian regression. These problems lead to analyzing objectives of the following form:
\begin{equation}\label{eq: ours}
    \Phi (G) = \min_{w\in \calS_w} \max_{v \in \calS_{v}} \sum_{\ell = 1}^k v_{\ell}^T G_\ell\Sigma_{\ell}^{\frac{1}{2}} w  + \psi(w,v)
\end{equation}

Note that \eqref{eq: ours} differs from \eqref{eq: theirs} in two substantial ways. First, \eqref{eq: ours} explicitly contains the covariance matrices, which is clearly crucial for the applications we mentioned. Second, note that in \eqref{eq: ours} the same $w$ is shared between each term, while in \eqref{eq: theirs} each $w_\ell$ is taken to be different. If, however, one were to still take each $w_\ell$ to be the same in \eqref{eq: theirs}, analyzing the resulting problem would reduce to the standard CGMT with a larger Gaussian matrix $G = [G_1, \dots, G_k]^T$.  

\end{remark}
\begin{remark}\label{rem: montanari}
The first arxiv version of the present manuscript considered only the less general case of $A_\ell = \Sigma_\ell^{\frac{1}{2}}$  being an arbitrary PSD matrix, $B_\ell = I_{n_k}$, $a_\ell = 0$ and $b_\ell = 0$ as this was enough to address the applications mentioned in Subsections \ref{sec:prelim:subsec:multisourceGaussian} and \ref{subsec:bin_class} . Concurrent to the posting of this paper on arxiv, the authors noticed the posting of \cite{jain2024scaling} which considers a different pair of processes to which CGMT can be applied. Inspired by combining this work and theirs, we note that our main Theorem \ref{thm: gcgmt} can be generalized even further.
    In terms of the terminology above, the processes designed in \cite{jain2024scaling} correspond to $A_\ell = I_d$ and $B_\ell = I_{n_\ell}$ and arbitrary $a_\ell, b_\ell$.
\end{remark} 

\section{Applications}\label{sec:Applications}

\subsection{Multi-source regression}\label{sec:Applications:subsec:MultiSourceRegression}

We recall the problem given in equation~\eqref{eq:GaussianRegression}. Also recall the definition of the Moreau envelope, for a function $f$ and matrix $T$:
\begin{equation} \label{def: scalMorenv}
    \calM_{T\ f(\cdot)}(x) = \min_{z} f(z) + \frac{1}{2}(x - z)^TT^{-1}(x - z),
\end{equation}
where $T^{-1}$ is a positive definite matrix. We also note that the optimal point of the optimization in \eqref{def: scalMorenv} is the proximal operator, denoted $\mathrm{prox}_{T\ f(\cdot)}(x)$.
The application of the generalized CGMT to the multi-source regression problem gives us the following theorem:

\begin{theorem}[Multi-source Generalization Error]\label{thm:multi-source-gen-error}
Consider the following scalarized optimization problem:
\begin{eqnarray}\label{eq:mult_source_ao}
     \min_{\xi, q\succeq 0}\max_{\beta,r\succeq 0} \sum_{l=1}^k \frac{\beta_l q_l}{2k} - \frac{\xi_lr_l}{2k} 
    + \frac{1}{nk}\calM_{\frac{q_l}{\beta_l}I\ \ell_l}\left(\nu_l - \xi_lh_l\right) \nwl
    -  \frac{1}{k^2}\frac{\beta_l^2}{n}\tr\Sigma_l\left(\frac{1}{k}\sum_{l'=1}^k\frac{r_{l'}}{\xi_{l'}}\Sigma_{l'}\right)^{-1} + \frac{1}{d}\calM_{A^{-1}\ R}\left(\theta^* - A^{-1}b\right)
\end{eqnarray}
Here $\xi, q, \beta, r\in \bbR^k$ are element-wise non-negative, and $h_l\in\bbR^n$ for $l=1,\ldots, k$ are i.i.d standard Gaussian vectors independent of each other. Furthermore, $A$ and $b$ are defined as follows:
\begin{align*}
    & A(r, \xi)\in\bbR^{d\times d} = \left(\frac{1}{k}\sum_{l=1}^k\frac{r_l}{\xi_k}\Sigma_l\right) \\
    & b(\beta)\in\bbR^d = \left(\frac{1}{k}\sqrt{\frac{d}{n}}\sum_{l=1}^k\beta_l\Sigma_l^{1/2} g_l\right).
\end{align*}
We denote by $\hat{r}_i, \hat{\xi}_i, \hat{\beta}_i, \hat{q}_i$ the optimal values of the scalarized optimization problem, and define $\hat{A} = A(\hat{r}, \hat{\xi})$ and $\hat{b} = b(\hat{\beta})$.
Then the generalization error (as defined in equation \ref{eq:multisource_gen_error}) for the multi-source Gaussian Regression problem is given by:
\begin{equation*}
     \calE_{gen} = \frac{1}{2k}\bbE\sum_{l=1}^k\norm{y_{l, new} - \frac{1}{\sqrt{d}}x_{l, new}^T\mathrm{prox}_{\hat{A}^{-1}\ R}(\theta^* - \hat A^{-1}\hat b)}
\end{equation*}
\end{theorem}


\subsection{Binary classification for general GMMs}

\label{sec: class}

In this section, we provide the precise asymptotic analysis of binary linear classification for data generated from a GMM with arbitrary covariances.
\begin{theorem}[Binary Classification Error]\label{th: classification master}
The classification error obtained from the solution of \eqref{eq:classification} is given as
$$
\frac{1}{2} Q\left(\frac{\gamma_1}{\theta_1}\right) + \frac{1}{2} Q\left(-\frac{\gamma_2}{\theta_2}\right)
$$
where $\gamma_i$ and $\theta_i$ are solutions of the following scalar optimization problem:
\begin{align*}
    &\min_{ \tau_i,\theta_i, \zeta_i\ge 0,\gamma_i} \max_{ \beta_i \geq 0, \eta_i}  \lambda \calM_{\lambda \ f(\Sigma^{-\frac{1}{2}})}\left(\Sigma^{-\frac{1}{2}}x\right)- \frac{x^T\Sigma^{-1}x}{2} \\
    & + \sum^2_{i=1} \left[\beta_i \frac{\tau_i}{2} +   \calM_{\frac{\tau_i}{\beta_i} \ \calL} \left(\gamma_i \mathds{1} -\theta_i h_i -z_i\right)  -  \frac{\theta_i\zeta_i}{2} -\eta_i\gamma_i\right]
\end{align*}
where $\Sigma = \sum_{i=1}^{2} \frac{\zeta_i}{\theta_i}\Sigma_i$ and $x = \sum_{i=1}^{2}\beta_i\Sigma_i^{\frac{1}{2}}g_i - \eta_i\mu_i$
\end{theorem}

In the special case of the $\ell_2$-norm square loss and regularization, we have the following simplified result:
\begin{theorem}\label{th: classification l2}
The classification error obtained from the solution of \eqref{eq:classification} when $\mathcal{L}(\cdot) = \|\cdot\|^2$ and $f(\cdot)=\|\cdot\|^2$  is given as
\begin{align*}
    \frac{1}{2} Q\left(\frac{\gamma_1+2}{\sqrt{\frac{8}{n} \tau_1^2 - \gamma_1^2}}\right) + \frac{1}{2} Q\left(\frac{-\gamma_2+2}{\sqrt{\frac{8}{n} \tau_2^2 - \gamma_2^2}}\right)
\end{align*}
where $\beta_i, \gamma_i$ and $\tau_i$ are solutions of the following scalar optimization problem:
\begin{align*}
     & \min_{\tau_i \ge 0} \max_{\beta_i \ge 0, \gamma_i} -\frac{1}{4}  \tr \Bigl[(\lambda I + \frac{n\beta_1}{4\tau_1} \Sigma_1 + \frac{n\beta_2}{4\tau_2} \Sigma_2)^{-1}  (\frac{n^2 \beta_1 \beta_2 \gamma_1 \gamma_2}{8 \tau_1 \tau_2}rI+\sum_{i=1}^2 \beta_i^2 \Sigma_i + \frac{n^2 \beta_i^2 \gamma_i^2}{16 \tau_i^2} I)\Bigr]  \\
     &+\sum_{i=1}^2 \left[\frac{\beta_i \tau_i}{2} - \frac{n\beta_i \gamma_i^2}{16\tau_i} - \frac{n \beta_i \gamma_i (-1)^{i+1}}{4\tau_i} - \frac{\beta_i^2}{4}\right]
\end{align*}
\label{thm: classif_l2}
\end{theorem}
The performances of various classification algorithms have been analyzed in depth using the standard CGMT for the case of scalar $\Sigma_1$ and $\Sigma_2$ (see, for example,  \citep{pmlr-v119-mignacco20a, thrampoulidis2020theoretical, akhtiamov2023regularized}). As a baby step towards developing a better understanding of the effects of the structure of $\Sigma_1$ and $\Sigma_2$ on the classification error, we apply Theorem \ref{thm: classif_l2} to the following explicitly defined model: 
\begin{align}\label{eq:SigmaModel}
     & \Sigma_1 = \sigma_1^2I + \nu_1 \nu_1^T ~~~\mbox{and}~~~ \Sigma_2 = \sigma_2^2I + \nu_2 \nu_2^T \\ \nonumber 
     & \text{ where } \sigma_1, \sigma_2, \sigma \in \mathbb{R} ~~~\mbox{and}~~~ \nu_1, \nu_2 \text{ are i.i.d. } \mathcal{N}(0, \sigma^2)
\end{align}
Plugging in these assumptions into Theorem \ref{thm: classif_l2} yields the following Corollary.
\begin{corollary}
The classification error obtained from the solution of \eqref{eq:classification} when $\mathcal{L}(\cdot) = \|\cdot\|^2, f(\cdot)=\|\cdot\|^2$ for $\Sigma_1$ and $\Sigma_2$ defined as above is given as:
\begin{align*}
    \frac{1}{2} Q\left(\frac{\gamma_1+2}{\sqrt{\frac{8}{n} \tau_1^2 - \gamma_1^2}}\right) + \frac{1}{2} Q\left(\frac{-\gamma_2+2}{\sqrt{\frac{8}{n} \tau_2^2 - \gamma_2^2}}\right)
\end{align*}
where $\beta_i, \gamma_i$ and $\tau_i$ are solutions of the following scalar optimization problem:
\begin{align*}
     & \min_{\tau_i \ge 0} \max_{\beta_i \ge 0, \gamma_i} (d-2)(\lambda  + \frac{n\beta_1}{4\tau_1} \sigma^2_1 + \frac{n\beta_2}{4\tau_2} \sigma^2_2 ) ^{-1} \left(\frac{n^2 \beta_1 \beta_2 \gamma_1 \gamma_2}{8 \tau_1 \tau_2}r+\sum_{i=1}^2 \left[\beta_i^2 \sigma^2_i + \frac{n^2 \beta_i^2 \gamma_i^2}{16 \tau_i^2}\right]\right)  \\
     & + (\lambda  + \frac{n\beta_1}{4\tau_1} (\sigma^2_1 + \sigma^2d) + \frac{n\beta_2}{4\tau_2} \sigma^2_2 ) ^{-1}   \left(\frac{n^2 \beta_1 \beta_2 \gamma_1 \gamma_2}{8 \tau_1 \tau_2}r+ \beta_1^2\sigma^2d+ \sum_{i=1}^2 \left[\beta_i^2 \sigma^2_i + \frac{n^2 \beta_i^2 \gamma_i^2}{16 \tau_i^2}\right]\right) \\
     & + (\lambda  + \frac{n\beta_1}{4\tau_1} \sigma^2_1  + \frac{n\beta_2}{4\tau_2} (\sigma^2_2 + \sigma^2d)) ^{-1}
    \left(\frac{n^2 \beta_1 \beta_2 \gamma_1 \gamma_2}{8 \tau_1 \tau_2}r+ \beta_2^2 \sigma^2d + \sum_{i=1}^2 \left[ \beta_i^2 \sigma^2_i + \frac{n^2 \beta_i^2 \gamma_i^2}{16 \tau_i^2}\right]\right) \\
     &+\sum_{i=1}^2 \left[\frac{\beta_i \tau_i}{2} - \frac{n\beta_i \gamma_i^2}{16\tau_i} - \frac{n \beta_i \gamma_i (-1)^{i+1}}{4\tau_i} - \frac{\beta_i^2}{4}\right]
\end{align*}
\label{cor: rk_1_noise}
\end{corollary}
\begin{remark}
        Since the first term in the objective above has $(d-2)$ in front of it, and the other two do not, one might approximate the objective by only the first term for large enough $d$. Indeed, even though the second and the third terms have some occurrences of $d$ as well, in both cases this quantity is present both in the numerator and the denominator of a certain fraction. This tells us that as $d$ grows large, we should not expect the performance to be very different from the case $\tilde{\Sigma}_1 = \sigma_1^2I, \tilde{\Sigma}_2 = \sigma_2^2I$, as this is what the first term corresponds to if we replace $(d-2)$ by $d$, again by the virtue of the fact that $d$ is large. Since this observation is not rigorous, we check it numerically in the next section. 
        \label{rem: rk_1_noise}
\end{remark}
To conclude, note that it is also feasible to study linear classification for GMMs via approaches based on AMP (see \citep{NEURIPS2021_543e8374}). Since the present paper is devoted to studying approaches based on Gaussian comparison inequalities, covering other methods goes beyond its scope. But in any case, the authors believe that developing new tools for tackling a problem is always of value, as it usually leads to new insights, such as Remark \ref{rem: rk_1_noise}, which has not been observed before in the literature to the best of the authors' knowledge.

\section{Numerical experiments}
\subsection{Multi-source regression plots}
To validate Theorem~\ref{thm:multi-source-gen-error} numerically we consider the specific case of all $\ell_l$ being the square loss and $R$ being square regularization parameterized by regularization strength $\lambda$, ie $R = \lambda\norm{\cdot}_2^2$. For this particular choice of loss and regularizer, the expressions are given in the supplement section~\ref{app:sec:mutlisourceGauss:subsec:l2l2}. The results for a variety of values of the regularization strengths $\lambda$ can be seen in Fig.~\ref{fig:lam_vary}, for several choices for the ratio $\frac{d}{n}$. The figure on the left shows the training error and the right shows the generalization error with respect to new data. The data has $k=3$ channels, and $n = 100$, and $d=50, 100, 150$ for each of the three curves. For $\Sigma_l$ we use the model described in equation~\eqref{eq:SigmaModel}, with $\sigma_1 = 0.5, \sigma_2 = 0.7, \sigma_3 = 0.3$ respectively and $\nu_i$ drawn from i.i.d standard Gaussian. For the choice of $\theta^*$ we choose a vector of all $1$s. Finally, the noise strength for each of the channels was chosen to be $\sigma_{\bnu, 1} = 0.1, \sigma_{\bnu, 2} = 0.2, \sigma_{\bnu, 3} = 0.3$.
In Fig.~\ref{fig:lam_vary}, the solid line denotes the values predicted by our alternative optimization, given in \eqref{eq:mult_source_ao} solved numerically using zeroth order gradient ascent/descent. The marks are given by a numerical simulation of the problem \eqref{eq:GaussianRegression} with the same parameters. As can be observed the predicted values match with a high degree of accuracy to their numerical equivalents, even while being in the finite regime. 
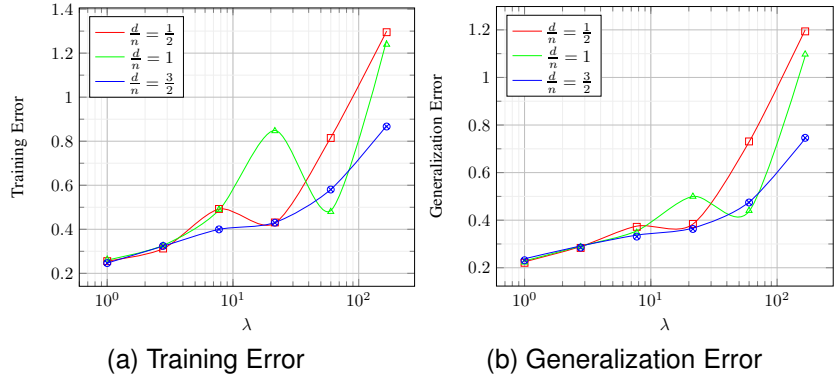
\begin{figure}[!h]
    \centering
    \subfloat [Training Error]{\label{subfig:lam_vary_train}
    \resizebox{0.3\textwidth}{!}{%
        \begin{tikzpicture}
        \begin{axis}[
          xmode=log,
          xlabel={$\lambda$},
          ylabel=Training Error,
          legend pos = north west,
          grid = both,
            minor tick num = 1,
            major grid style = {lightgray},
            minor grid style = {lightgray!25},]
        \addplot[ smooth, thin, red] table[ y=train, x=lam]{Data/lamvary_n=100_m=50_k=3_theory.dat};
        \addlegendentry{$\frac{d}{n} = \frac{1}{2}$}
        \addplot[ smooth, thin, green] table[ y=train, x=lam]{Data/lamvary_n=100_m=100_k=3_theory.dat};
        \addlegendentry{$\frac{d}{n} = 1$}
        \addplot[ smooth, thin, blue] table[ y=train, x=lam]{Data/lamvary_n=100_m=150_k=3_theory.dat};
        \addlegendentry{$\frac{d}{n} = \frac{3}{2}$}
        \addplot[color = red, mark = square, mark size = 2pt, only marks] table[ y=train, x=lam, y error=trainstd]{Data/lamvary_n=100_m=50_k=3_exp.dat};
        \addplot[color = green, mark = triangle, mark size = 2pt, only marks] table[ y=train, x=lam, y error=trainstd]{Data/lamvary_n=100_m=100_k=3_exp.dat};
        \addplot[color = blue, mark = otimes, mark size = 2pt, only marks] table[ y=train, x=lam, y error=trainstd]{Data/lamvary_n=100_m=150_k=3_exp.dat};
        \end{axis}
    \end{tikzpicture}}}
\subfloat [Generalization Error] {
    \label{subfig:lam_vary_gen}
    \resizebox{0.3\textwidth}{!}{%
        \begin{tikzpicture}
        \begin{axis}[
          xmode=log,
          xlabel={$\lambda$},
          ylabel=Generalization Error,
          legend pos = north west,
          grid = both,
            minor tick num = 1,
            major grid style = {lightgray},
            minor grid style = {lightgray!25},]
        \addplot[ smooth, thin, red] table[ y=gen, x=lam]{Data/lamvary_n=100_m=50_k=3_theory.dat};
        \addlegendentry{$\frac{d}{n} = \frac{1}{2}$}
        \addplot[ smooth, thin, green] table[ y=gen, x=lam]{Data/lamvary_n=100_m=100_k=3_theory.dat};
        \addlegendentry{$\frac{d}{n} = 1$}
        \addplot[ smooth, thin, blue] table[ y=gen, x=lam]{Data/lamvary_n=100_m=150_k=3_theory.dat};
        \addlegendentry{$\frac{d}{n} = \frac{3}{2}$}
        \addplot[color = red, mark = square, mark size = 2pt, only marks] table[ y=gen, x=lam, y error=trainstd]{Data/lamvary_n=100_m=50_k=3_exp.dat};
        \addplot[color = green, mark = triangle, mark size = 2pt, only marks] table[ y=gen, x=lam, y error=trainstd]{Data/lamvary_n=100_m=100_k=3_exp.dat};
        \addplot[color = blue, mark = otimes, mark size = 2pt, only marks] table[ y=gen, x=lam, y error=trainstd]{Data/lamvary_n=100_m=150_k=3_exp.dat};
        \end{axis}
    \end{tikzpicture}}}
    \caption{Training and Generalization Error for the Multi-source Gaussian Regression}
    \label{fig:lam_vary}
\end{figure}
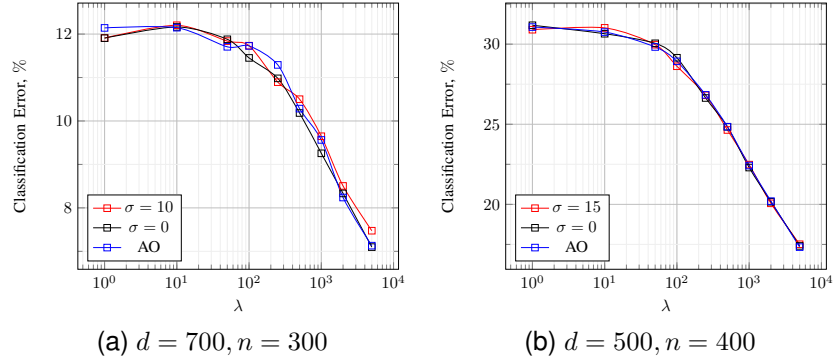
\begin{figure}[!h]
    \centering
   \subfloat [$d = 700, n = 300$] {
    \label{subfig:binary_d=700}
    \resizebox{0.3\textwidth}{!}{%
        \begin{tikzpicture}
        \begin{axis}[
          xmode=log,
          xlabel={$\lambda$},
          ylabel=Classification Error$\text{, }\%$,
          legend pos = south west,
          grid = both,
            minor tick num = 1,
            major grid style = {lightgray},
            minor grid style = {lightgray!25},
          ]
        \addplot[mark = square, smooth, thin, red] 
        table[ y=POR1, x=lam]{Binary_Data/newd700.dat};
        \addlegendentry{$\sigma = 10$}
        \addplot[mark = square, smooth, thin, black] table[ y=PO, x=lam]{Binary_Data/newd700.dat};
        \addlegendentry{ $ \sigma = 0$}
        \addplot[mark=square, smooth, thin, blue] table[ y=AO, x=lam]{Binary_Data/newd700.dat};
        \addlegendentry{AO} 
        \end{axis}
    \end{tikzpicture}   
        } 
    }
\subfloat [$d = 500, n = 400$] {
    \label{subfig:bin_d=500}
    \resizebox{0.3\textwidth}{!}{%
        \begin{tikzpicture}
        \begin{axis}[
          xmode=log,
          xlabel={$\lambda$},
          ylabel=Classification Error$\text{, }\%$,
          legend pos = south west,
          grid = both,
            minor tick num = 1,
            major grid style = {lightgray},
            minor grid style = {lightgray!25},
          ]
        \addplot[mark = square, smooth, thin, red] table[ y=POR1, x=lam]{Binary_Data/newd500.dat};
        \addlegendentry{$\sigma = 15$}
        \addplot[mark = square, smooth, thin, black] table[ y=PO, x=lam]{Binary_Data/newd500.dat};
        \addlegendentry{$\sigma = 0$}
        \addplot[mark = square, smooth, thin, blue] table[ y=AO, x=lam]{Binary_Data/newd500.dat};
        \addlegendentry{AO} 
        \end{axis}
    \end{tikzpicture}
        } 
    }
       \caption{The error for binary classification for GMMs}
    \label{fig:bin}
\end{figure}

\subsection{Binary classification for general GMMs}

We conducted experiments to validate the results of Section \ref{sec: class} numerically as presented in Fig. \ref{fig:bin}. To do so, we fixed $\calL(\cdot) = f(\cdot) = \|\cdot\|_2^2$ and defined $\Sigma_1, \Sigma_2$ following the notation from \eqref{eq:SigmaModel}. 
We took $\sigma_1 = \sigma_2 = 3, \sigma = 10, d = 700 ,n = 300, r = 0.9$ for the plot on the LHS and $\sigma_1 = \sigma_2 = 5, \sigma = 15, d = 500, n = 400, r = 0.8$ for the plot on the RHS, solved \eqref{eq:classification} directly for $\lambda \in \{1,10,50,100,250,500,1000,2000,5000\}$ using CVXPY \citep{diamond2016cvxpy, agrawal2018rewriting}, evaluated the corresponding classification errors via \eqref{eq: error} and plotted them in red with legends "$\sigma = 10$" and "$\sigma = 15$"  respectively. After that, we solved the scalarized objective from Corollary \ref{cor: rk_1_noise} by noticing that $\beta_1 = \beta_2$, $\tau_1 = \tau_2$ and $\gamma_1 = -\gamma_2$ due to the distribution of $(\Sigma_1, \Sigma_2)$ being exchangeable, running a grid search over $\tau: = \tau_1 = \tau_2$ and using scipy.optimize.minimize to perform optimization over the remaining parameters, which we then used to evaluate the corresponding classification error predicted by Corollary \ref{cor: rk_1_noise} for each $\lambda$ and plotted it in blue with a legend "AO". Finally, as an additional sanity check, we decided to plot the classification error corresponding to $\tilde{\Sigma}_1 = \sigma_1^2I$ and $\tilde{\Sigma}_2 = \sigma_2^2I$ because according to Remark \ref{rem: rk_1_noise} we expect it to not deviate too much from the classification error for the GMM with the corresponding $\Sigma_1, \Sigma_2$. The results of the latter procedure are depicted in black with a legend "$\sigma = 0$". All three described lines match very closely in both settings as expected.

\section{Conclusion and future work}

We presented a new pair of Gaussian processes that satisfy Gordon's comparison conditions and thereby developed a generalization of the Convex Gaussian Min-Max Theorem. We further demonstrated its use in two examples, namely multi-source regression and binary classification for Gaussian data with arbitrary covariance matrices.  Future work would be to see if further pairs of comparable processes can be found. Of particular interest are those where $G_\ell$ and $G_{\ell'}$ do not have to be independent for $\ell \ne \ell'$. 

\newpage

\section{Appendix}
\section{Proof of Theorem \ref{thm: gcgmt}}
We will need the following technical lemma:
\begin{lemma}
    Under the notation from Theorem \ref{thm: gcgmt}, let $\delta >0$ and $\calS^{\delta}_w, \calS^{\delta}_v$ be arbitrary finite $\delta$ - nets for $\calS_w$ and $\calS_v$ respectively and define
\begin{align*}
    &\tilde{\Phi}(G) = \min_{v \in \calS_v} \max_{w\in \calS_w} \sum_{\ell = 1}^k v_{\ell}^T G_\ell\Sigma_{\ell}^{\frac{1}{2}} w- \psi(w,v) \\
&\tilde{\Phi}^{\delta}(G) = \min_{v \in \calS^{\delta}_v} \max_{w\in \calS^{\delta}_w} \sum_{\ell = 1}^k v_{\ell}^T G_\ell\Sigma_{\ell}^{\frac{1}{2}} w - \psi(w,v) \\
&\tilde{\phi}(g,h) = \min_{v\in \calS_v} \max_{w\in \calS_w} \sum_{\ell = 1}^k \|v_\ell\|_2 g_{\ell}^T \Sigma_{\ell}^{\frac{1}{2}}w + \|\Sigma_{\ell}^{\frac{1}{2}}w\|_2 h_{\ell}^T v_\ell - \psi(w,v) \\
&\tilde{\phi}^{\delta}(g,h) = \min_{v\in \calS^{\delta}_v} \max_{w\in \calS^{\delta}_w} \sum_{\ell = 1}^k \|v_\ell\|_2 g_{\ell}^T \Sigma_{\ell}^{\frac{1}{2}}w + \|\Sigma_{\ell}^{\frac{1}{2}}w\|_2 h_{\ell}^T v_\ell - \psi(w,v) \\
&\Phi^{\delta} (G) = \min_{w\in \calS^{\delta}_w} \max_{v \in \calS^{\delta}_{v}} \sum_{\ell = 1}^k v_{\ell}^T G_\ell\Sigma_{\ell}^{\frac{1}{2}} w  + \psi(w,v) \\
&\phi^{\delta} (g, h) = \min_{w\in \calS^{\delta}_w} \max_{v\in \calS^{\delta}_v} \sum_{\ell = 1}^k \|v_\ell\|_2 g_{\ell}^T \Sigma_{\ell}^{\frac{1}{2}}w + \|\Sigma_{\ell}^{\frac{1}{2}}w\|_2 h_{\ell}^T v_\ell + \psi(w,v)
\end{align*}
Then the following holds:
\begin{align*}
    & \lim_{\delta \to 0} \bbP(\chi^{\delta} < \chi) = 1 \text{ if } \chi \in \{\Phi(G), \tilde{\Phi}(G)\} ~~~\mbox{and}~~~ \lim_{\delta \to 0} \bbP(\chi < \chi^{\delta}) = 1 \text{ if }  \chi \in\{\phi(g,h), \tilde{\phi}(g,h)\} \\
\end{align*}
\label{lm: delta_nets}
\end{lemma}
\begin{proof}
    Define 
    \begin{align*}
        &D = \max_{w \in \calS_w, v \in \calS_v}\{\|w\|, \|v\|\} \\
        & L = \max_{(w,v) \ne (w',v')} \frac{|\psi(w,v) - \psi(w',v')|}{\|w - w'\|+ \|v - v'\|} \\
    \end{align*}
    We then have:
    \begin{align*}
        & \Phi^{\delta}(G) -  \Phi(G) = \min_{w\in \calS^{\delta}_w} \max_{v \in \calS^{\delta}_{v}} \sum_{\ell = 1}^k v_{\ell}^T G_\ell\Sigma_{\ell}^{\frac{1}{2}} w  + \psi(w,v) - \min_{w\in \calS_w} \max_{v \in \calS_{v}} \sum_{\ell = 1}^k v_{\ell}^T G_\ell\Sigma_{\ell}^{\frac{1}{2}} w  + \psi(w,v) \le  \\
        & \le \min_{w\in \calS^{\delta}_w} \max_{v \in \calS^{\delta}_{v}} \sum_{\ell = 1}^k v_{\ell}^T G_\ell\Sigma_{\ell}^{\frac{1}{2}} w  + \psi(w,v) - \min_{w\in \calS_w} \max_{v \in \calS^{\delta}_{v}} \sum_{\ell = 1}^k v_{\ell}^T G_\ell\Sigma_{\ell}^{\frac{1}{2}} w  + \psi(w,v) \le \\
        & \le \max_{v \in \calS^{\delta}_{v}} \sum_{\ell = 1}^k v_{\ell}^T G_\ell\Sigma_{\ell}^{\frac{1}{2}} w^{\delta}_*  + \psi(w^{\delta}_*,v) - \max_{v \in \calS^{\delta}_{v}} \sum_{\ell = 1}^k v_{\ell}^T G_\ell\Sigma_{\ell}^{\frac{1}{2}} w_*  + \psi(w_*,v) \le \\
        & \le \max_{v \in \calS^{\delta}_{v}} \left[\sum_{\ell = 1}^k v_{\ell}^T G_\ell\Sigma_{\ell}^{\frac{1}{2}} w^{\delta}_*  + \psi(w^{\delta}_*,v)\right] - \left[\sum_{\ell = 1}^k v_{\ell}^T G_\ell\Sigma_{\ell}^{\frac{1}{2}} w_*  + \psi(w_*,v)\right] = \\
        & = \left[\sum_{\ell = 1}^k {v_{\ell}}_*^T G_\ell\Sigma_{\ell}^{\frac{1}{2}} w^{\delta}_* + \psi(w^{\delta}_*,v_*)\right] - \left[\sum_{\ell = 1}^k {v_{\ell}}_*^T G_\ell\Sigma_{\ell}^{\frac{1}{2}} w_*  + \psi(w_*,v_*)\right]  \le \left(\sum^k_{\ell = 1}\|G_\ell\|_{op}\|\Sigma_{\ell}^{\frac{1}{2}}\|_{op} \right) D\delta + L\delta \le \\
        & \le \left(\sum^k_{\ell = 1}\sqrt{d} + \sqrt{n_\ell} + \frac{1}{\sqrt{\delta}} \right) D\sigma\delta + L\delta \text{ with probability at least } (1 - 2e^{\frac{-1}{4\delta}})^k \text{ due to Gaussian concentration, where }   
        \end{align*}
        \begin{align*}
        & \sigma = \max_\ell \|\Sigma^{\frac{1}{2}}_{\ell}\|_{op}\\
        & w_* = \argmin_{w\in \calS_w} \max_{v \in \calS^{\delta}_{v}} \sum_{\ell = 1}^k v_{\ell}^T G_\ell\Sigma_{\ell}^{\frac{1}{2}} w  + \psi(w,v) \\
        & w^{\delta}_* \in  \calS^{\delta}_w \text{ is such that } \|w^{\delta}_* - w_*\| \le \delta \\
        & v_* = \argmax_{v \in \calS^{\delta}_{v}} \left[\sum_{\ell = 1}^k v_{\ell}^T G_\ell\Sigma_{\ell}^{\frac{1}{2}} w^{\delta}_*  + \psi(w^{\delta}_*,v)\right] - \left[\sum_{\ell = 1}^k v_{\ell}^T G_\ell\Sigma_{\ell}^{\frac{1}{2}} w_*  + \psi(w_*,v)\right] \end{align*}
    We also have:
    \begin{align*}     
    &\phi^{\delta}(g,h) -  \phi(g,h) =   \min_{w\in \calS^{\delta}_w} \max_{v\in \calS^{\delta}_v} \sum_{\ell = 1}^k \|v_\ell\|_2 g_{\ell}^T \Sigma_{\ell}^{\frac{1}{2}}w + \|\Sigma_{\ell}^{\frac{1}{2}}w\|_2 h_{\ell}^T v_\ell + \psi(w,v) - \\
    & - \min_{w\in \calS_w} \max_{v\in \calS_v} \sum_{\ell = 1}^k \|v_\ell\|_2 g_{\ell}^T \Sigma_{\ell}^{\frac{1}{2}}w + \|\Sigma_{\ell}^{\frac{1}{2}}w\|_2 h_{\ell}^T v_\ell + \psi(w,v)  \ge \min_{w\in \calS^{\delta}_w} \max_{v\in \calS^{\delta}_v} \sum_{\ell = 1}^k \|v_\ell\|_2 g_{\ell}^T \Sigma_{\ell}^{\frac{1}{2}}w + \|\Sigma_{\ell}^{\frac{1}{2}}w\|_2 h_{\ell}^T v_\ell + \psi(w,v) - \\
    & - \min_{w\in \calS^{\delta}_w} \max_{v\in \calS_v} \sum_{\ell = 1}^k \|v_\ell\|_2 g_{\ell}^T \Sigma_{\ell}^{\frac{1}{2}}w + \|\Sigma_{\ell}^{\frac{1}{2}}w\|_2 h_{\ell}^T v_\ell + \psi(w,v) \ge \\
    & \ge \min_{w\in \calS^{\delta}_w}  \left[\max_{v\in \calS^{\delta}_v}\sum_{\ell = 1}^k \|v_\ell\|_2 g_{\ell}^T \Sigma_{\ell}^{\frac{1}{2}}w + \|\Sigma_{\ell}^{\frac{1}{2}}w\|_2 h_{\ell}^T v_\ell + \psi(w,v) \right] - \\
    & - \left[\max_{v\in \calS_v}\sum_{\ell = 1}^k \|v_\ell\|_2 g_{\ell}^T \Sigma_{\ell}^{\frac{1}{2}}w + \|\Sigma_{\ell}^{\frac{1}{2}}w\|_2 h_{\ell}^T v_\ell + \psi(w,v)\right] \ge \\
    & \ge \left[\sum_{\ell = 1}^k \|{v^{\delta}_\ell}_*\|_2 g_{\ell}^T \Sigma_{\ell}^{\frac{1}{2}}w_* + \|\Sigma_{\ell}^{\frac{1}{2}}w_*\|_2 h_{\ell}^T {v^{\delta}_\ell}_* + \psi(w_*,v^{\delta}_*) \right] - \\
    & -\left[\sum_{\ell = 1}^k \|{v_\ell}_*\|_2 g_{\ell}^T \Sigma_{\ell}^{\frac{1}{2}}w_* + \|\Sigma_{\ell}^{\frac{1}{2}}w_*\|_2 h_{\ell}^T {v_\ell}_* + \psi(w_*,v_*) \right] \ge \\
    & \ge - \left(\sum^k_{\ell = 1} (\|g_\ell\| + \|h_\ell\|) \|\Sigma_\ell^{\frac{1}{2}}\|_{op} \right)D\delta - L\delta \ge  -\left(\sum^k_{\ell = 1}\sqrt{d} + \sqrt{n_\ell} + \frac{1}{\sqrt{\delta}} \right) D\sigma\delta - L\delta\\
    & \text{ with probability at least } (1 - 2e^{\frac{-1}{4\delta}})^k \text{ due to Gaussian concentration, where} \\
    & w_* = \argmin_{w\in \calS^{\delta}_w}  \left[\max_{v\in \calS^{\delta}_v}\sum_{\ell = 1}^k \|v_\ell\|_2 g_{\ell}^T \Sigma_{\ell}^{\frac{1}{2}}w + \|\Sigma_{\ell}^{\frac{1}{2}}w\|_2 h_{\ell}^T v_\ell + \psi(w,v) \right] - \\
    & - \left[\max_{v\in \calS_v}\sum_{\ell = 1}^k \|v_\ell\|_2 g_{\ell}^T \Sigma_{\ell}^{\frac{1}{2}}w + \|\Sigma_{\ell}^{\frac{1}{2}}w\|_2 h_{\ell}^T v_\ell + \psi(w,v)\right] \\
    & v_* = \argmax_{v\in \calS_v} \sum_{\ell = 1}^k \|{v_\ell}\|_2 g_{\ell}^T \Sigma_{\ell}^{\frac{1}{2}}w_* + \|\Sigma_{\ell}^{\frac{1}{2}}w_*\|_2 h_{\ell}^T {v_\ell} + \psi(w_*,v) \\
    & v^{\delta}_* \in \calS^{\delta}_v \text{ is such that }\|v^{\delta}_* - v_*\| \le \delta
\end{align*}
We obtain in a similar fashion: 
    \begin{align*}
        & \tilde{\Phi}^{\delta}(G) -  \tilde{\Phi}(G) = \min_{v \in \calS^{\delta}_{v}} \max_{w\in \calS^{\delta}_w}  \sum_{\ell = 1}^k v_{\ell}^T G_\ell\Sigma_{\ell}^{\frac{1}{2}} w  - \psi(w,v) - \min_{v \in \calS_{v}} \max_{w\in \calS_w} \sum_{\ell = 1}^k v_{\ell}^T G_\ell\Sigma_{\ell}^{\frac{1}{2}} w  - \psi(w,v) \le  \\
        & \le \min_{v \in \calS^{\delta}_{v}} \max_{w\in \calS^{\delta}_w} \sum_{\ell = 1}^k v_{\ell}^T G_\ell\Sigma_{\ell}^{\frac{1}{2}} w  - \psi(w,v) - \min_{v \in \calS_v} \max_{w\in \calS^{\delta}_w} \sum_{\ell = 1}^k v_{\ell}^T G_\ell\Sigma_{\ell}^{\frac{1}{2}} w  - \psi(w,v) \le \\
        & \le \max_{w \in \calS^{\delta}_w} \sum_{\ell = 1}^k {v^{\delta}_{\ell}}_*^T G_\ell\Sigma_{\ell}^{\frac{1}{2}} w  - \psi(w,v^{\delta}_*) - \max_{w \in \calS^{\delta}_w} \sum_{\ell = 1}^k v_*^T G_\ell\Sigma_{\ell}^{\frac{1}{2}} w  - \psi(w,v_*) \le  \\
        & \le \max_{w \in \calS^{\delta}_w} \left[\sum_{\ell = 1}^k {v^{\delta}_{\ell}}_*^T G_\ell\Sigma_{\ell}^{\frac{1}{2}} w  - \psi(w,v^{\delta}_*)\right] - \left[\sum_{\ell = 1}^k v_*^T G_\ell\Sigma_{\ell}^{\frac{1}{2}} w  - \psi(w,v_*) \right] = \\
        \end{align*}
        \begin{align*}
        & = \left[\sum_{\ell = 1}^k {v^{\delta}_{\ell}}_*^T G_\ell\Sigma_{\ell}^{\frac{1}{2}} w_* -\psi(w_*,v^{\delta}_*)\right] - \left[\sum_{\ell = 1}^k {v_{\ell}}_*^T G_\ell\Sigma_{\ell}^{\frac{1}{2}} w_*  - \psi(w_*,v_*)\right]  \le \left(\sum^k_{\ell = 1}\|G_\ell\|_{op}\|\Sigma_{\ell}^{\frac{1}{2}}\|_{op} \right) D\delta + L\delta \le \\
        & \le \left(\sum^k_{\ell = 1}\sqrt{d} + \sqrt{n_\ell} + \frac{1}{\sqrt{\delta}} \right) D\sigma\delta + L\delta \text{ with probability at least } (1 - 2e^{\frac{-1}{4\delta}})^k \text{ due to Gaussian concentration, where } 
        \end{align*}
        \begin{align*}
        & \sigma = \max_\ell \|\Sigma^{\frac{1}{2}}_{\ell}\|_{op}\\
        & v_* = \argmin_{v\in \calS_v} \max_{w \in \calS^{\delta}_{w}} \sum_{\ell = 1}^k v_{\ell}^T G_\ell\Sigma_{\ell}^{\frac{1}{2}} w  + \psi(w,v) \\
        & v^{\delta}_* \in  \calS^{\delta}_v \text{ is such that } \|v^{\delta}_* - v_*\| \le \delta \\
        & w_* = \argmax_{w \in \calS^{\delta}_w} \left[\sum_{\ell = 1}^k {v^{\delta}_{\ell}}_*^T G_\ell\Sigma_{\ell}^{\frac{1}{2}} w  - \psi(w,v^{\delta}_*)\right] - \left[\sum_{\ell = 1}^k v_*^T G_\ell\Sigma_{\ell}^{\frac{1}{2}} w  - \psi(w,v_*) \right]
    \end{align*}
    Finally, derived in a similar fashion:
    \begin{align*}
    &\tilde{\phi}^{\delta}(g,h) -  \tilde{\phi}(g,h) =   \min_{v\in \calS^{\delta}_v} \max_{w\in \calS^{\delta}_w} \sum_{\ell = 1}^k \|v_\ell\|_2 g_{\ell}^T \Sigma_{\ell}^{\frac{1}{2}}w + \|\Sigma_{\ell}^{\frac{1}{2}}w\|_2 h_{\ell}^T v_\ell - \psi(w,v) - \\
    & - \min_{v\in \calS_v} \max_{w\in \calS_w} \sum_{\ell = 1}^k \|v_\ell\|_2 g_{\ell}^T \Sigma_{\ell}^{\frac{1}{2}}w + \|\Sigma_{\ell}^{\frac{1}{2}}w\|_2 h_{\ell}^T v_\ell - \psi(w,v) \ge \\ 
    & \ge \min_{v\in \calS^{\delta}_v} \max_{w\in \calS^{\delta}_w} \sum_{\ell = 1}^k \|v_\ell\|_2 g_{\ell}^T \Sigma_{\ell}^{\frac{1}{2}}w + \|\Sigma_{\ell}^{\frac{1}{2}}w\|_2 h_{\ell}^T v_\ell - \psi(w,v) - \\
    & - \min_{v\in \calS^{\delta}_v} \max_{w\in \calS_w} \sum_{\ell = 1}^k \|v_\ell\|_2 g_{\ell}^T \Sigma_{\ell}^{\frac{1}{2}}w + \|\Sigma_{\ell}^{\frac{1}{2}}w\|_2 h_{\ell}^T v_\ell - \psi(w,v) \ge \\
    & \ge \min_{v \in \calS^{\delta}_v}  \left[\max_{w\in \calS^{\delta}_w}\sum_{\ell = 1}^k \|v_\ell\|_2 g_{\ell}^T \Sigma_{\ell}^{\frac{1}{2}}w + \|\Sigma_{\ell}^{\frac{1}{2}}w\|_2 h_{\ell}^T v_\ell - \psi(w,v) \right] - \\
    & - \left[\max_{w\in \calS_w}\sum_{\ell = 1}^k \|v_\ell\|_2 g_{\ell}^T \Sigma_{\ell}^{\frac{1}{2}}w + \|\Sigma_{\ell}^{\frac{1}{2}}w\|_2 h_{\ell}^T v_\ell - \psi(w,v)\right] \ge \\
    & \ge \left[\sum_{\ell = 1}^k \|{v_\ell}_*\|_2 g_{\ell}^T \Sigma_{\ell}^{\frac{1}{2}}w^{\delta}_* + \|\Sigma_{\ell}^{\frac{1}{2}}w^{\delta}_*\|_2 h_{\ell}^T {v_\ell}_* + \psi(w^{\delta}_*,v_*) \right] - \\
    & -\left[\sum_{\ell = 1}^k \|{v_\ell}_*\|_2 g_{\ell}^T \Sigma_{\ell}^{\frac{1}{2}}w_* + \|\Sigma_{\ell}^{\frac{1}{2}}w_*\|_2 h_{\ell}^T {v_\ell}_* + \psi(w_*,v_*) \right] \ge \\
    & \ge - \left(\sum^k_{\ell = 1} (\|g_\ell\| + \|h_\ell\|) \|\Sigma_\ell^{\frac{1}{2}}\|_{op} \right)D\delta - L\delta \ge  -\left(\sum^k_{\ell = 1}\sqrt{d} + \sqrt{n_\ell} + \frac{1}{\sqrt{\delta}} \right) D\sigma\delta - L\delta\\
    & \text{ with probability at least } (1 - 2e^{\frac{-1}{4\delta}})^k \text{ due to Gaussian concentration, where} \\
    & v_* = \argmin_{v\in \calS^{\delta}_v}  \left[\max_{w\in \calS^{\delta}_w}\sum_{\ell = 1}^k \|v_\ell\|_2 g_{\ell}^T \Sigma_{\ell}^{\frac{1}{2}}w + \|\Sigma_{\ell}^{\frac{1}{2}}w\|_2 h_{\ell}^T v_\ell - \psi(w,v) \right] - \\
    & - \left[\max_{w\in \calS_w}\sum_{\ell = 1}^k \|v_\ell\|_2 g_{\ell}^T \Sigma_{\ell}^{\frac{1}{2}}w + \|\Sigma_{\ell}^{\frac{1}{2}}w\|_2 h_{\ell}^T v_\ell - \psi(w,v)\right] \\
    & w_* = \argmax_{w\in \calS_w} \sum_{\ell = 1}^k \|{v_\ell}_*\|_2 g_{\ell}^T \Sigma_{\ell}^{\frac{1}{2}}w + \|\Sigma_{\ell}^{\frac{1}{2}}w\|_2 h_{\ell}^T {v_\ell}_* + \psi(w,v_*) \\
    & w^{\delta}_* \in \calS^{\delta}_w \text{ is such that }\|w^{\delta}_* - w_*\| \le \delta
\end{align*}
    Taking the limit $\delta \to 0$ in all inequalities above, we arrive at the desired results. 
\end{proof}
We are now ready to prove Theorem \ref{thm: gcgmt} modulo Lemma \ref{lm: process_with_sigmas}:
\begin{proof}
{\bf 1.} Note that $|\Phi (G) - c | > \eta$ holds if and only if either $\Phi (G) > c + \eta$ or $\Phi (G) < c -\eta$ and, moreover, the latter two events are disjoint since $\eta > 0$. The same description is true for $|\phi (g, h) - c| > \eta$. Thus, it suffices to show the following two inequalities:  
\begin{align}
    & \bbP(\Phi (G) < c - \eta) \le 2^k \bbP(\phi (g, h) < c - \eta)  \\
    & \bbP(-\Phi (G) < -c  - \eta) \le 2^k \bbP(-\phi (g, h) < -c - \eta)
    \label{ineq: -Phi}
\end{align}


To simplify the second inequality above, swap min and max in the definition of $\Phi(G)$ due to the convexity-concavity of $\psi(w,v)$ and note that: 
\begin{align*}
    -\Phi(G) = - \max_{v\in \calS_v} \min_{w\in \calS_w} v^T \begin{pmatrix}
    G_1 & 0  &\cdots & 0 \\
    0 & G_2  & \cdots & 0 \\
    \vdots & & \ddots & \vdots\\
    0 & 0 & \cdots & G_k  
\end{pmatrix} \begin{pmatrix}
    \Sigma_1^{\frac{1}{2}} \\
    \Sigma_2^{\frac{1}{2}} \\
    \vdots \\
    \Sigma_k^{\frac{1}{2}}
\end{pmatrix} w + \psi(w,v) \\
= \min_{v \in \calS_v} \max_{w\in \calS_w} v^T \begin{pmatrix}
    -G_1 & 0  &\cdots & 0 \\
    0 & -G_2  & \cdots & 0 \\
    \vdots & & \ddots & \vdots\\
    0 & 0 & \cdots & -G_k  
\end{pmatrix} \begin{pmatrix}
    \Sigma_1^{\frac{1}{2}} \\
    \Sigma_2^{\frac{1}{2}} \\
    \vdots \\
    \Sigma_k^{\frac{1}{2}}
\end{pmatrix}w - \psi(w,v)
\end{align*}
Hence, using that $G_1, \dots, G_k$ are i.i.d. $\mathcal{N}(0,1)$, we obtain:
\begin{align*}
    \bbP(-\Phi(G) < -c-\eta) = \bbP(\tilde{\Phi}(G) < -c-\eta) 
\end{align*}
where 
\begin{align*}
    \tilde{\Phi}(G) = \min_{v \in \calS_v} \max_{w\in \calS_w} v^T \begin{pmatrix}
    G_1 & 0  &\cdots & 0 \\
    0 & G_2  & \cdots & 0 \\
    \vdots & & \ddots & \vdots\\
    0 & 0 & \cdots & G_k  
\end{pmatrix} \begin{pmatrix}
    \Sigma_1^{\frac{1}{2}} \\
    \Sigma_2^{\frac{1}{2}} \\
    \vdots \\
    \Sigma_k^{\frac{1}{2}}
\end{pmatrix}w - \psi(w,v)
\end{align*}

Also remark:
\begin{align*}
    & \bbP(-\phi (g, h) < - c - \eta) = \bbP \left(-\min_{w\in \calS_w} \max_{v\in \calS_v} \sum_{\ell = 1}^k \|v_\ell\|_2 g_{\ell}^T \Sigma_{\ell}^{\frac{1}{2}}w + \|\Sigma_{\ell}^{\frac{1}{2}}w\|_2 h_{\ell}^T v_\ell + \psi(w,v) < - c - \eta \right) = \\
    & = \bbP \left(\max_{w\in \calS_w} \min_{v\in \calS_v} \sum_{\ell = 1}^k \|v_\ell\|_2 (-g_{\ell})^T \Sigma_{\ell}^{\frac{1}{2}}w + \|\Sigma_{\ell}^{\frac{1}{2}}w\|_2 (-h_{\ell})^T v_\ell - \psi(w,v) < - c - \eta \right) = \\
    & = \bbP \left(\max_{w\in \calS_w} \min_{v\in \calS_v} \sum_{\ell = 1}^k \|v_\ell\|_2 g_{\ell}^T \Sigma_{\ell}^{\frac{1}{2}}w + \|\Sigma_{\ell}^{\frac{1}{2}}w\|_2 h_{\ell}^T v_\ell - \psi(w,v) < - c - \eta \right) \le \\
    & \le \bbP \left(\min_{v\in \calS_v} \max_{w\in \calS_w} \sum_{\ell = 1}^k \|v_\ell\|_2 g_{\ell}^T \Sigma_{\ell}^{\frac{1}{2}}w + \|\Sigma_{\ell}^{\frac{1}{2}}w\|_2 h_{\ell}^T v_\ell - \psi(w,v) < - c - \eta \right) = \bbP(\tilde{\phi}(g,h) < - c - \eta) \\
    &  \text{ where } \tilde{\phi}(g,h) = \min_{v\in \calS_v} \max_{w\in \calS_w} \sum_{\ell = 1}^k \|v_\ell\|_2 g_{\ell}^T \Sigma_{\ell}^{\frac{1}{2}}w + \|\Sigma_{\ell}^{\frac{1}{2}}w\|_2 h_{\ell}^T v_\ell - \psi(w,v)
\end{align*}

Denoting $C = c - \eta$ and $\tilde{C} = - c - \eta$, we then have to show the following two inequalities: 

\begin{align}
    & \bbP(\Phi (G) <  C) \le 2^k \bbP(\phi (g, h) < C)  \\
    & \bbP(\tilde{\Phi} (G) < \tilde{C}) \le 2^k \bbP(\tilde{\phi} (g, h) < \tilde{C})
\end{align}

Note that according to Lemma \ref{lm: delta_nets}, it is enough to prove the desired inequalities for any sequence of $\delta$-nets $\calS^{\delta}_w, \calS^{\delta}_v$ with $\delta \to 0$. Indeed, we would then have $$\bbP(\Phi (G) >  C) \ge \bbP(\Phi^{\delta} (G) >  C) \ge 2^k\bbP(\phi^{\delta} (g,h) >  C) \ge 2^k\bbP(\phi(g,h) >  C)$$ as well as the same chain of inequalities for $\tilde{\Phi}(G), \tilde{\phi}(g,h)$ and $\tilde{C}$. Moreover, we can assume that each $\calS^{\delta}_v$ is of the form $\calS^{\delta}_v = \calS^{\delta}_{v_1} \times \dots \times \calS^{\delta}_{v_k}$. Denote $\calS^{\delta}_w = \{w_i\}_{i \in I}, \calS^{\delta}_{v_1} = \{u_{j_1}\}_{j_1 \in J_1} , \dots, \calS^{\delta}_{v_k} = \{u_{j_k}\}_{j_k \in J_k}$, $t_{ij} = C - \psi(w_i, v_j), \tilde{t}_{ji} = \tilde{C} + \psi(w_i, v_j)$. Construct processes $X_{ij},\tilde{X}_{ji}, Y_{ij}, \tilde{Y}_{ji}$ as in the statement of Lemma \ref{lm: process_with_sigmas} for $I, J_1, \dots, J_k$ introduced in the previous sentence. Applying Lemma \ref{lm: process_with_sigmas} and Theorem \ref{thm: comparison} to $X, Y, t$ and $\tilde{X}, \tilde{Y},\tilde{t}$ we obtain:

\begin{align*}
\bbP(\bigcap_{i \in I} \bigcup_{j \in J} [ Y_{ij} \ge t_{ij} ]) \ge \bbP(\bigcap_{i \in I} \bigcup_{j \in J} [ X_{ij} \ge t_{ij}]) \\
\bbP(\bigcap_{j \in J} \bigcup_{i \in I} [ \tilde{Y}_{ji} \ge \tilde{t}_{ji} ]) \ge \bbP(\bigcap_{j \in J} \bigcup_{i \in I} [ \tilde{X}_{ji} \ge \tilde{t}_{ji}])
\end{align*}

We then deduce:

\begin{align*}
     \bbP(\min_{i \in I} \max_{j \in J} {X_{ij} + \psi(w_i,v_j) < C}) = 1 - \bbP(\bigcap_{i\in I} \bigcup_{j \in J} [ X_{ij} + \psi(w_i,v_j) \ge C ]) = 1 - \bbP(\bigcap_{i\in I} \bigcup_{j \in J} [ X_{ij} \ge t_{ij} ]) \ge \\
    \ge 1 - \bbP(\bigcap_{i\in I} \bigcup_{j \in J} [ Y_{ij} \ge t_{ij} ]) =  1 - \bbP(\bigcap_{i\in I} \bigcup_{j \in J} [ Y_{ij} + \psi(w_i,v_j) \ge C ]) = \bbP(\min_{i \in I} \max_{j \in J} {Y_{ij} + \psi(w_i,v_j) <  C}) \\
    \bbP(\min_{j \in J} \max_{i \in I} {\tilde{X}_{ji} - \psi(w_i,v_j) < \tilde{C}}) = 1 - \bbP(\bigcap_{j\in J} \bigcup_{i \in I} [\tilde{X}_{ji} - \psi(w_i,v_j) \ge  \tilde{C} ]) = 1 - \bbP(\bigcap_{j\in J} \bigcup_{i \in I} [ \tilde{X}_{ji} \ge \tilde{t}_{ji}]) \ge \\
    \ge 1 - \bbP(\bigcap_{j\in J} \bigcup_{i \in I} [ \tilde{Y}_{ji} \ge \tilde{t}_{ji}]) = 1 - \bbP(\bigcap_{j\in J} \bigcup_{i \in I} [\tilde{Y}_{ji} - \psi(w_i,v_j) \ge  \tilde{C} ])   = \bbP(\min_{j \in J} \max_{i \in I} {\tilde{Y}_{ji} - \psi(w_i,v_j) < \tilde{C}}) 
\end{align*}

Note that
\begin{align*}
    &\bbP(\min_{i \in I} \max_{j \in J} {X_{ij} + \psi(w_i,v_j) <  C})  = \bbP(\phi^{\delta}(g,h) < C)\\
    &\bbP(\min_{j \in J} \max_{i \in I} {\tilde{X}_{ji} - \psi(w_i,v_j) < \tilde{C}}) = \bbP(\tilde{\phi}^{\delta}(g,h) < \tilde{C})
\end{align*}
Hence, it suffices to show the following two inequalities: 
\begin{align*}
    &\bbP(\Phi^{\delta}(G) < C) \ge \frac{1}{2^k}\bbP(\min_{i \in I} \max_{j \in J} {Y_{ij} + \psi(w_i,v_j) < C})  \\
    &\bbP(\tilde{\Phi}^{\delta}(G) < \tilde{C}) \ge \frac{1}{2^k} \bbP(\min_{j \in J} \max_{i \in I} {\tilde{Y}_{ji} - \psi(w_i,v_j) < \tilde{C}}) 
\end{align*}
The latter inequalities follow directly by conditioning on the event $\gamma_1, \dots, \gamma_k > 0$, which happens with probability $\frac{1}{2^k}$.  

{\bf 2.} As we know from the proof of the previous part of the theorem, the following upper bound takes place:

\begin{align*}
    \bbP(\Phi(G) < \barphi + \delta ) \le 2^k \bbP(\phi < \barphi + \delta ) \le 2^k \epsilon
\end{align*}

Since $\calS^c$ is compact, we can apply the same $\delta$-net argument to deduce the following inequality as well:

\begin{align*}
    \bbP( \Phi_{\calS^c} (G) < \barphi_{\calS^c} - \delta) \le 2^k \bbP( \phi_{\calS^c} < \barphi_{\calS^c} - \delta ) \le 2^k \epsilon
\end{align*}

Combining the results above implies:

\begin{align*}
    \bbP(\{\Phi(G) > \barphi + \delta\} \cup \{ \Phi_{\calS^c} (G) < \barphi_{\calS^c} - \delta\} ) \le 2^{k+1} \epsilon
\end{align*}

Which yields: 
\begin{align*}
    1 -  2^{k+1} \epsilon < \bbP(\{\Phi(G) \le \barphi + \delta\} \cap \{ \Phi_{\calS^c} (G) \ge \barphi_{\calS^c} - \delta\} )
\end{align*}
By assumption $\barphi_{\calS^c} > \barphi + 3 \delta$. Thus:
\begin{align*}
    1 -  2^{k+1} \epsilon < \bbP(\{\Phi(G) \le \barphi + \delta\} \cap \{ \Phi_{\calS^c} (G) \ge \barphi_{\calS^c} - \delta\} ) \le \bbP(\Phi(G) < \Phi_{\calS^c}(G) - \delta) 
\end{align*}
Therefore:
\begin{align*}
    1 -  2^{k+1} \epsilon < \bbP( w_{\Phi(G)} \in \calS)
\end{align*}

\end{proof}

We proceed to prove Lemma \ref{lm: process_with_sigmas} now:

\begin{proof}
    Recall the definitions of the Gaussian processes we have to compare, where $v_j$ stands as a shorthand notation for $(v_{j_1}, \dots, v_{j_k})^T \in \mathbb{R}^{n_1} \times \dots \times \mathbb{R}^{n_k}$:

$$Y_{ij} = {v_j}^T \begin{pmatrix}
    G_1 & 0  &\cdots & 0 \\
    0 & G_2  & \cdots & 0 \\
    \vdots & & \ddots & \vdots\\
    0 & 0 & \cdots & G_k  
\end{pmatrix} \begin{pmatrix}
    \Sigma_1^{\frac{1}{2}} \\
    \Sigma_2^{\frac{1}{2}} \\
    \vdots \\
    \Sigma_k^{\frac{1}{2}}
\end{pmatrix}w_i + \sum^k_{\ell = 1}\gamma_\ell\|v_{j_\ell}\|\|\Sigma_\ell^{\frac{1}{2}}w_i\| ~~~\mbox{and}~~~ X_{ij} = \sum^k_{\ell=1} \|v_{j_\ell}\|w_i^T\Sigma^{\frac{1}{2}}_\ell g_\ell + {v_{j_\ell}}^Th_\ell\|\Sigma^{\frac{1}{2}}_\ell w_i\|$$
 
Calculating the covariances is straightforward using that $G_\ell$ and $G_{\ell'}$ are independent whenever $\ell \ne \ell'$:

$$ \bbE[Y_{ij}Y_{i'j'}] = \sum^k_{\ell=1} {v_{j_\ell}}^Tv_{j'_\ell} w_i^T \Sigma_\ell w_{i'} + \|v_{j_\ell}\|\|v_{j'_\ell}\|\|\Sigma_\ell^{\frac{1}{2}}w_i\|\|\Sigma_\ell^{\frac{1}{2}}w_{i'}\|$$
$$ \bbE[X_{ij}X_{i'j'}] = \sum^k_{\ell=1} \|v_{j_\ell}\|\|v_{j'_\ell}\| w_i^T \Sigma_\ell w_{i'} + {v_{j_\ell}}^Tv_{j'_\ell}\|\Sigma_\ell^{\frac{1}{2}}w_i\|\|\Sigma_\ell^{\frac{1}{2}}w_{i'}\|$$

Hence, the third assumption of Gordon’s Gaussian comparison theorem follows from Cauchy-Schwartz: 
$$\bbE[Y_{ij}Y_{i'j'} - X_{ij}X_{i'j'} ] = \sum_{\ell=1}^k ({v_{j_\ell}}^Tv_{j'_\ell} - \|v_{j_\ell}\|\|v_{j'_\ell}\|)({w_i}^T \Sigma_\ell w_{i'} -\|\Sigma_\ell^{\frac{1}{2}}w_i\|\|\Sigma_\ell^{\frac{1}{2}}w_{i'}\|) \ge 0$$

Moreover, if either $i=i'$ or $j=j'$ then the expression above vanishes, implying the first and the second assumptions as well. The proof is almost identical for the processes $\tilde{X}$ and $\tilde{Y}$.

\end{proof}

\section{Concentration}

\label{sec: concentration}

\begin{lemma}
    $\phi(g,h)$ is a $\sigma\sqrt{2} R_wR_v$- Lipschitz function of $(g, h)$, where 
    \begin{align*}
        & \sigma = \max_\ell \|\Sigma^{\frac{1}{2}}_{\ell}\|_{op}\\
        & R_w = \max_{w \in \calS_w}\{\|w\|\} \\
        & R_v = \max_{v \in \calS_v}\{\|v\|\}
    \end{align*}
\end{lemma}

\begin{proof}

Let $g_1, g_2 \in \mathbb{R}^{d \times \ell }$ and $h_1, h_2 \in \mathbb{R}^{n \times \ell}$ be $\ell$-tuples of i.i.d. Gaussian vectors. 

Define: 
\begin{align*}
    &(w_1,v_1)= \argmin_{w\in S_w} \max_{v\in S_v}\sum^k_{\ell=1} \|v_{\ell}\|w^T\Sigma^{\frac{1}{2}}_\ell {g_1}_{\ell} + v_{\ell}^T{h_1}_{\ell}\|\Sigma^{\frac{1}{2}}_\ell w\| + \psi(w,v) \\
    &v_2 = \argmax_{v\in S_v}\sum^k_{\ell=1} \|v_\ell\|w_1^T\Sigma^{\frac{1}{2}}_\ell g_{2\ell} + v_\ell^Th_{2\ell}\|\Sigma^{\frac{1}{2}}_\ell w_1\| + \psi(w_1,v)
\end{align*}
We then have:
\begin{align*}
    \phi(g_2,h_2) \le \sum^k_{\ell=1}\|v_2\|w_1^T\Sigma^{\frac{1}{2}}_\ell {g_2}_\ell + {v_2}_\ell^T{h_2}_\ell\|\Sigma^{\frac{1}{2}}_\ell w_1\| + \psi(w_1,v_2) \\ 
    \phi(g_1,h_1) \ge  \sum^k_{\ell=1}\|v_2\|w_1^T\Sigma^{\frac{1}{2}}_\ell {g_1}_\ell + {v_2}_\ell^T{h_1}_\ell\|\Sigma^{\frac{1}{2}}_\ell w_1\| + \psi(w_1,v_2)
\end{align*}
Therefore:
\begin{align*}
    \phi(g_2,h_2) - \phi(g_1,h_1) &\le \sum^k_{\ell=1}\|v_2\|w_1^T\Sigma^{\frac{1}{2}}_\ell(g_{2\ell}-g_{1\ell}) + {v_2}_\ell^T(h_{2\ell}-h_{1\ell})\|\Sigma^{\frac{1}{2}}_\ell w_1\| \le \\
    &\le \sum^k_{\ell=1}\|v_2\|\|w_1^T\Sigma^{\frac{1}{2}}_\ell\|\|g_{1\ell}-g_{2\ell}\| + \|\Sigma^{\frac{1}{2}}_\ell w_1\| \|v_2\| \|h_{1\ell}-h_{2\ell}\| = \\
  &= \sum^k_{\ell=1}\|v_2\|\|w_1^T\Sigma^{\frac{1}{2}}_\ell\|(\|g_{1\ell}-g_{2\ell}\|+\|h_{1\ell}-h_{2\ell}\|) \le  \\
    \end{align*}
    \begin{align*}
    &\le \sqrt{2\sum^k_{\ell=1} \|v_2\|^2\|\Sigma^{\frac{1}{2}}_\ell w_1\|^2} \sqrt{\sum^k_{\ell=1} \|g_{1\ell}-g_{2\ell}\|^2 + \|h_{1\ell}-h_{2\ell}\|^2} \le \\ &\le \sigma\sqrt{2} R_w R_v \sqrt{\sum^k_{\ell=1} \|g_{1\ell}-g_{2\ell}\|^2 + \|h_{1\ell}-h_{2\ell}\|^2} 
\end{align*}
\end{proof}

\begin{corollary}
    Under the notation from the lemma above, $$\bbP(|\phi(g,h) - \bbE \phi(g,h)| > \epsilon) \le exp(-\frac{\epsilon^2}{4\sigma^2 R^2_wR^2_v})$$
\end{corollary}
\begin{proof}

Follows immediately from the lemma above and Theorem 5.6 in \cite{DBLP:books/daglib/0035704}. 
\end{proof}

\section{Applications}
\subsection{Multi-Source Gaussian Regression}
\subsubsection{Assumptions}
The following assumptions are necessary for the Multi-source Gaussian Regression Results to hold.
\begin{itemize}
    \item[A1] The functions $\ell_l$, for all $l=1,\ldots, k$ and $R$ are proper, lower semi-continuous, and convex. Furthermore, the function $\sum_{l=1}^k \ell_l + R$ should be coercive, by which we mean:
    \begin{eqnarray*}
        \lim_{\norm{\btheta}_2 \rightarrow\infty}\left(\sum_{l=1}^k \ell_l + R \right)(\theta) = +\infty
    \end{eqnarray*}
    \item[A2] the Regularization function $R$ must satisfy the following scaling condition. For any vector $x$ such that $\norm{x}_2 \leq c\sqrt{d}$ for some constant $c$, this implies that
    \begin{eqnarray*}
        \sup_{z \in \partial R(z)}\norm{z} \leq C\sqrt{d}
    \end{eqnarray*}
    where $\partial R$ is the subdifferential of $R$, and $C$ is some constant.
    \item[A3] the noise $\nu_l$, for all $l=1, \ldots, k$, is a zero mean sub-Gaussian random vector, with variance $\sigma_{\bnu, l}$.
\end{itemize}

We recall the setup discussed in section~\ref{sec:prelim:subsec:multisourceGaussian}
\begin{equation}
    \min_{\btheta\in\bbR^{d}} \frac{1}{nk}\sum_{l=1}^k \ell_l\left(y_l - \frac{1}{\sqrt{d}}X_{l}\theta \right) + R(\theta),
\end{equation}
Where $y_l \in \bbR^n, X_{l}\in\bbR^{n \times d}, \theta\in\bbR^d$. We recall that $X_l$ are Gaussian matrices, defined by $G_k\Sigma^{1/2}_l$, where $G \in \bbR^{n\times d}$ has i.i.d Gaussian elements, and $\Sigma_l\in\bbR^{d\times d}$ are PSD covariance matrices. Finally, recalling the assumption that $\by_k = \frac{1}{\sqrt{d}}X\theta^* + \nu_k$ where $\bnu_k$ is i.i.d zero mean noise, with variance $\sigma_{\nu, k}$

We first define $e = \theta - \theta^*$ as the error vector and rewrite our optimization as
\begin{eqnarray*}
    \min_{e\in\bbR^{d}} \frac{1}{nk}\sum_{l=1}^k \ell_l(\nu_l - \frac{1}{\sqrt{d}}G_l\Sigma^{1/2}\be) + R(e + \theta^*)
\end{eqnarray*}
We then let $z_l = \nu_l - G_l\Sigma^{1/2}_le$ and let $\lambda_l$ be a Lagrange multiplier 
\begin{eqnarray}\label{app:eq:mutlisource_error_vec_optimizations}
    \min_{e\in\bbR^{d}, Z\in\bbR^{n\times d}} \max_{\Lambda\in\bbR^{n\times d}} \sum_{l=1}^k \frac{1}{nk}\lambda_l^T(z_l  - \nu_l + \frac{1}{\sqrt{d}}G_l\Sigma_k^{1/2}e) + \frac{1}{nk}\ell_l(z_l) + R(\be + \theta^*).
\end{eqnarray}

Before we can apply our application of the generalized CGMT we must restrict our optimization to compact and convex sets. We demonstrate this in the following lemma.

\begin{lemma}[Compact and Convex Sets]
    Let $\hat{e}$, $\hat{Z}$, $\hat{\Lambda}$ be the optimal solution to \eqref{app:eq:mutlisource_error_vec_optimizations}. There exist positive constants $C_{e}, C_{Z}, C_{\Lambda}$ such that:
    \begin{align*}
        \Pr\left(\norm{\hat{e}}_2 \leq C_{e}\sqrt{d} \right) \xrightarrow[d\rightarrow \infty]{P} 1\\
        \Pr\left(\norm{\hat{Z}}_2 \leq C_{Z}\sqrt{n} \right) \xrightarrow[n\rightarrow \infty]{P} 1\\
        \Pr\left(\norm{\hat{\Lambda}}_2 \leq C_{\Lambda}\sqrt{n} \right) \xrightarrow[n\rightarrow \infty]{P} 1
    \end{align*}
\end{lemma}
\begin{proof}
    We first show that the set of minimizers for $e$ is non-empty and compact. We assume that the sum $f(e) = \sum_{l=1}^k \ell(\nu_l - m^{-1/2}G_l\Sigma^{1/2}_le) + R(e + \theta^*)$ is a coercive, proper, and lower-semi continuous functions. As the function is proper, there exists at least a single value $e_0$ such that $f(\be_0) < \infty$. Furthermore, by the fact that $f$ is coercive, there exists an $\epsilon \in\bbR\geq 0$ such that for every $e\in \bbR^{d\times k}$ for which $\norm{e - e_0} \geq \epsilon$ implies that $f(e) \geq f(e_0)$.

    We then define
    \begin{eqnarray*}
        S = \set{e \in \bbR^{d}|\ \norm{e - e_0}_2 \leq \epsilon},
    \end{eqnarray*}
    we can then see that $S\cap \bbR^{d\times k} \neq \emptyset$ is nonempty and compact. As such, there there exists an $\hat{e} \in S$ such that $f(\hat{e}) = \inf_{e \in S} f(e) \leq f(e_0)$. From this we can conclude that $f(\hat{e}) \in \inf_{e \in \bbR^{d}} f(e)$, and that the set of minimizers is bounded.

    Having determined that the set of minimizers is a compact set with at least one element, we can conclude that there exists a constant $C_{\bE}$ which is independent of the ambient dimension $d$, such that
    \begin{eqnarray*}
        \norm{\hat{e}}_2 \leq C_{e}\sqrt{d}.
    \end{eqnarray*}

    We now recall equation~\eqref{app:eq:mutlisource_error_vec_optimizations}, and examine the optimality condition:
    \begin{eqnarray*}
        \nabla_{\lambda_l}:\ & \nu_l - \frac{1}{\sqrt{m}}G_l\Sigma_l^{1/2}e = z_l \nwl
        \nabla_{z_l}:\ & \lambda_l \in \partial \ell_l(z_l)\qquad \nwl
        \nabla_{e}:\ & \sum_{l=1}^k\Sigma_l^{1/2}G_{l}^T\lambda_l  \in \left.\partial R(x)\right|_{x = e + \theta^*}.
    \end{eqnarray*}
    From the optimality condition of $\lambda_k$ we can see that 
    \begin{eqnarray*}
        \norm{z_l} \leq \norm{\nu_l}_2 + \frac{1}{\sqrt{d}}\norm{G_l}_2\norm{\Sigma_l^{1/2}}_2\norm{e}
    \end{eqnarray*}
    We note that $\nu_l$ is a sub-Gaussian random variable, as such its norm can asymptotically be bounded by a $C\sqrt{n}$ for some constant $C$. Similarly $\frac{1}{\sqrt{m}}\norm{G_l}$ is bounded in operator norm by a constant asymptotically by standard random matrix theory results (see for example \citep{papaspiliopoulos2020high}[corollary 7.3.3]). By making use of the bound on $\norm{e}$ we can find that each $z_l$ is bounded in norm. From this we can see that there exists some constant $C_{Z}$ such that
    \begin{eqnarray*}
        \Pr\left(\norm{\hat{Z}}_2 \leq C_{Z}\sqrt{n} \right) \xrightarrow[n\rightarrow \infty]{P} 1
    \end{eqnarray*}
    Finally, the scaling condition on $R$ implies the existence of a constant $C_{\lambda}$.
\end{proof}

We now define the following sets
\begin{eqnarray*}
    \calS_{e} = \set{e \in \bbR^{d}|\ \norm{e} \leq C_{e}\sqrt{d}}\nwl
    \calS_{Z} = \set{Z \in \bbR^{n\times k}|\ \norm{Z} \leq C_{Z}\sqrt{n}}\nwl
    \calS_{\Lambda} = \set{\Lambda \in \bbR^{n\times k}|\ \norm{\Lambda} \leq C_{\Lambda}\sqrt{n}}, 
\end{eqnarray*}
and note that our optimization problem can be expressed as:
\begin{eqnarray*}
    \min_{e\in\calS_{e}, Z\in\calS_{Z}} \max_{\Lambda\in\calS_{\Lambda}} \sum_{l=1}^k \frac{1}{nk}\lambda_l^T(z_l  - \nu_l + \frac{1}{\sqrt{d}}G_l\Sigma_k^{1/2}e) + \frac{1}{nk}\ell_l(z_l) + R(e + \theta^*)\nwl
    \overset{def}{=}\min_{e \in \calS_{e}} \calP(e, G_1, \cdots, G_k)
\end{eqnarray*}
We note that we now have a min-max optimization over compact and convex sets in the form required for the application of theorem \ref{thm: gcgmt}. We apply the generalized CGMT to $\calP$, from which we obtain:
\begin{eqnarray*}
    \min_{e\in\calS_{e}} \left[\min_{Z\in\calS_{Z}} \max_{\Lambda\in\calS_{\Lambda}} \sum_{l=1}^k \frac{1}{nk\sqrt{d}}\norm{\lambda_l}g_l^T\Sigma_l^{1/2}e + \frac{1}{nk\sqrt{d}}\norm{\Sigma^{1/2}_le}h_l^T\lambda_l \right. \nwl
    \left.+ \frac{1}{nk}\lambda_l^Tz_l - \frac{1}{nk}\lambda_l^T\nu_{l} + \frac{1}{nk}\ell_l(z_l) + R(e + \theta^*)\right]\nwl
    \overset{def}{=} \min_{e\in\calS_{e}} \calA(e, g_1, \ldots, g_k, h_1, \ldots, h_k)
\end{eqnarray*}

We can now note that if for a fixed value of $e$, if in the asymptotic limit of $n,m \rightarrow\infty$ that $\calA \rightarrow \bar\calA$ converges to some fixed value, the CGMT guarantees that $\calP \rightarrow \bar\calA$. This implies pointwise convergence on the set $\calS_{e}$, we can further establish uniform convergence by noting that $\calS_{e}$ is compact, and that $\calP$ is Lipschitz with respect to $e$, implying equicontinuity. Classical results from analysis guarantee uniform convergence in these cases. As such we drop the outer parenthesis and consider the joint minimization over $e, Z$ in the analysis of $\calA$.
\begin{eqnarray*}
    \min_{Z\in\calS_{Z}, e\in\calS_{e}} \max_{\Lambda\in\calS_{\Lambda}} \sum_{l=1}^k \frac{1}{nk\sqrt{d}}\norm{\lambda_l}g_l^T\Sigma_l^{1/2}e + \frac{1}{nk\sqrt{d}}\norm{\Sigma^{1/2}_le}h_l^T\lambda_l \nwl
    + \frac{1}{nk}\lambda_l^Tz_l - \frac{1}{nk}\lambda_l^T\nu_{l} + \frac{1}{nk}\ell_l(z_l) + R(\be + \theta^*)
\end{eqnarray*}
We now define $p_l = \Sigma^{1/2}_le$, and reintroduce this constraint using the Lagrange multiplier $\mu_l$. We can note for each each $p_l$ that $\norm{p_l}\leq \sqrt{\sigma_{max}(\Sigma_l)}\norm{e}$, by noting that $\Sigma_l$ has finite maximum singular value, and the bounds on $\norm{\be}$, we can conclude that there exists a compact set $\calS_{P} = \set{P\in\bbR^{d\times k}|\ \norm{P}_2 \leq C_{P}\sqrt{d}}$, where $P$ has columns $p_l$:
\begin{eqnarray*}
    \min_{Z\in\calS_{Z}, e\in\calS_{e}, P\in \calS_{P}} \max_{\Lambda\in\calS_{\Lambda},M} \sum_{l=1}^k \frac{1}{nk\sqrt{d}}\norm{\lambda_l}g_l^Tp_l + \frac{1}{nk\sqrt{d}}\norm{p_l}h_l^T\lambda_l \nwl
    + \frac{1}{nk}\lambda_l^Tz_l - \frac{1}{nk}\lambda_l^T\nu_{l} + \frac{1}{nk}\ell_l(z_l) + \frac{1}{dk}\mu_l^T(p_l - \Sigma^{1/2}_le) + R(e + \theta^*) 
\end{eqnarray*}
where $M\in\bbR^{d\times k}$ has columns $\mu_l$.

We now let $\beta_l = \frac{1}{\sqrt{n}}\norm{\Lambda_l}$ and solve over $\lambda_l$, as shorthand we will define $\beta \in\bbR^k$, with elements $\beta_l$.
\begin{eqnarray*}
    \min_{Z\in\calS_{Z}, e\in\calS_{e}, P\in \calS_{P}} \max_{\beta\succeq 0,M} \sum_{l=1}^k \frac{\beta_l}{k\sqrt{nd}}g_l^Tp + \frac{\beta_l}{k\sqrt{n}}\norm{\frac{1}{\sqrt{d}}\norm{p_l}_2h_l + z_l - \nu_l}_2 \nwl
    + \frac{1}{nk}\ell_l(z_l) + \frac{1}{dk}\mu_l^T(p_l - \Sigma^{1/2}e) + R(e + \theta^*)
\end{eqnarray*}
We now interchange the order of min and max, let $\xi_l = \frac{1}{\sqrt{d}}\norm{p_l}_2$, and solve over $P$. We obtain:
\begin{eqnarray*}
     \max_{\beta\succeq 0,M}\min_{Z\in\calS_{Z}, e\in\calS_{e}, \xi\succeq 0} \sum_{l=1}^k  \frac{\beta_l}{\sqrt{n}k}\norm{\xi_lh_l + z_l - \nu_l}_2 -\frac{\xi_l}{k\sqrt{d}}\norm{\sqrt{\frac{d}{n}}\beta_l\bg_l + \mu_l}_2 \nwl
    + \frac{1}{nk}\ell_l(z_l) - \frac{1}{dk}\mu_l^T\Sigma^{1/2}e + R(e + \theta^*)
\end{eqnarray*}

Recalling the square-root trick, where for any vector $a$, we can note that $\norm{a}_2 = \minl_{q>0} \frac{q}{2} + \frac{\norm{a}^2}{2q}$. We apply the square root trick twice to the optimization above with parameters $q_l$ and $r_l$
\begin{eqnarray*}
     \max_{\beta,r\succeq 0,M}\min_{Z\in\calS_{Z}, e\in\calS_{e}, \xi, q\succeq 0} \sum_{l=1}^k \frac{\beta_l q_l}{2k} - \frac{\xi_lr_l}{2k} + \frac{\beta_l}{2nkq_l}\norm{\xi_lh_l + z_l - \nu_l}^2_2 -\frac{\xi_l}{2kdr_l}\norm{\sqrt{\frac{d}{n}}\beta_lg_l + \mu_l}_2^2 \nwl
    + \frac{1}{nk}\ell_l(z_l) - \frac{1}{kd}\mu_l^T\Sigma^{1/2}e + R(e + \theta^*)
\end{eqnarray*}
We can now recognize the Moreau envelopes over $\ell_l$:
\begin{eqnarray*}
     \max_{\beta,r\succeq 0,M}\min_{e\in\calS_{e}, \xi, q\succeq 0} \sum_{l=1}^k \frac{\beta_l q_l}{2k} - \frac{\xi_lr_l}{2k} -\frac{\xi_l}{2kmr_l}\norm{\sqrt{\frac{d}{n}}\beta_lg_l + \mu_l}_2^2 \nwl
    + \frac{1}{nk}\calM_{\frac{q_l}{\beta_l}\ \ell_l}\left(\nu_l - \xi_lh_l\right) - \frac{1}{kd}\mu_l^T\Sigma^{1/2}_le + R(e + \theta^*)
\end{eqnarray*}
We once again interchange the order of min and max and now solve over $\mu_l$. We can find that the optimal solution is given by:
\begin{eqnarray*}
    \hat\mu_l = -\sqrt{\frac{d}{n}}\beta_lg_l  - \frac{r_l}{\xi_l}\Sigma^{1/2}_le
\end{eqnarray*}
Substituting in this value we obtain:
\begin{eqnarray*}
     \min_{\xi, q\succeq 0}\max_{\beta,r\succeq 0} \sum_{l=1}^k \frac{\beta_l q_l}{2k} - \frac{\xi_lr_l}{2k} 
    + \frac{1}{nk}\calM_{\frac{q_l}{\beta_l}\ \ell_l}\left(\nu_l - \xi_lh_l\right) \nwl
    + \min_{e\in\calS_{e}}\frac{1}{2d}e^T\left(\frac{1}{k}\sum_{l=1}^k\frac{r_l}{\xi_k}\Sigma_l\right)e + \frac{1}{d}\left(\frac{1}{k}\sum_{l=1}^k\beta_l\sqrt{\frac{d}{n}}\Sigma_l^{1/2}g_l\right)^Te_l + R(e + \theta^*)
\end{eqnarray*}
We can now complete the square over $e$, from which we find that 
\begin{eqnarray*}
     \min_{\xi, q\succeq 0}\max_{\bbeta,\br\succeq 0} \sum_{l=1}^k \frac{\beta_l q_l}{2k} - \frac{\xi_lr_l}{2k} 
    + \frac{1}{nk}\calM_{\frac{q_l}{\beta_l}\ \ell_l}\left(\nu_l - \xi_lh_l\right) - \frac{1}{2kd}b^TA^{-1}b \nwl
    + \min_{e\in\calS_{e}}\frac{1}{2d}\left(e + A^{-1}b\right)^TA\left(e + A^{-1}b\right)+ R(e + \theta^*)
\end{eqnarray*}
Where we have defined $A = \left(\frac{1}{k}\sum_{l=1}^k\frac{r_l}{\xi_k}\Sigma_l\right)$ and $b = \left(\frac{1}{k}\sum_{l=1}^k\beta_l\sqrt{\frac{d}{n}}\Sigma_l^{1/2}g_l\right)$. We can note that $b^TA^{-1}b$ will concentrate to 
\begin{eqnarray*}
    \frac{1}{k^2}\sum_{l=1}^k\frac{\beta_l^2d}{n}\tr\Sigma_l\left(\frac{1}{k}\sum_{l'=1}^k\frac{r_{l'}}{\xi_{l'}}\Sigma_{l'}\right)^{-1}
\end{eqnarray*}

We can now return $\theta = e + \theta^*$ as the optimization variable, and then recognize the Moreau envelope over $R$, we obtain the optimization:
\begin{eqnarray*}
     \min_{\bxi, \bq\succeq \bzero}\max_{\bbeta,\br\succeq 0} \sum_{l=1}^k \frac{\beta_l q_l}{2k} - \frac{\xi_lr_l}{2k} 
    + \frac{1}{nk}\calM_{\frac{q_l}{\beta_l}\ \ell_l}\left(\nu_l - \xi_lh_l\right) \nwl
    -  \frac{1}{k^2}\sum_{l=1}^k\frac{\beta_l^2}{n}\tr\Sigma_l\left(\frac{1}{k}\sum_{l'=1}^k\frac{r_{l'}}{\xi_{l'}}\Sigma_{l'}\right)^{-1} + \frac{1}{d}\calM_{A^{-1}\ R}\left(\theta^* - A^{-1}b\right)
\end{eqnarray*}

\subsubsection{$\ell_2^2$ loss, and $\ell_2^2$ regularization}\label{app:sec:mutlisourceGauss:subsec:l2l2}
We consider the loss functions $\frac{1}{2}\norm{\cdot}_2^2$. For the regularization we will consider the following function
\begin{eqnarray*}
    R(\theta) = \frac{\lambda}{2d}\norm{\theta}_2^2
\end{eqnarray*}
where $\lambda$ is a regularization parameter. We will assume that the noise $\nu_l$ will be distributed normally, by $\nu_l \sim\calN(0, \sigma_{\nu, l}^2I)$. We can find the following results:

\paragraph{$\norm{\cdot}_2$ loss} In this case we have:
\begin{eqnarray*}
    \calM_{\frac{q_l}{\beta_l}\ \norm{\cdot}_2^2}(\nu_l - \xi_lh_l) = \min_{z_l}\frac{1}{2}\norm{z_l}_2^2 + \frac{\beta_l}{2q_l}\norm{\nu_l -  \xi_lh_l-z_l}_2^2
\end{eqnarray*}
We can find the optimal solution to be given by:
\begin{eqnarray*}
    \hat{z_l} = \frac{\beta_l}{\beta_l + q_l}(\nu_l - \xi_lh_l)
\end{eqnarray*}
Similarly we can find the optimal value to be:
\begin{eqnarray*}
\calM_{\frac{q_l}{\beta_l}\ \norm{\cdot}_2^2}(\nu_l - \xi_lh_l) = \frac{1}{2}\norm{\frac{\beta_l}{\beta_l + q_l}(\nu_l - \xi_lh_l)}_2^2 + \frac{\beta_l}{2q_l}\norm{\frac{q_l}{\beta_l + q_l}\nu_l -  \frac{q_l\xi_l}{\beta_l + q_l}h_l}_2^2
\end{eqnarray*}
This concentrates on its expected value, given by:
\begin{eqnarray*}
\bbE\calM_{\frac{q_l}{\beta_l}\ \norm{\cdot}_2^2}(\nu_l - \xi_lh_l) = \frac{n}{2}\frac{\beta_l}{\beta_l + q_l}(\sigma_{\nu, l}^2 + \xi_l^2)
\end{eqnarray*}

\paragraph{Regularization Function} We can find that the regularization function in this case take the form
\begin{eqnarray*}
    \calM_{A^{-1}\ R}\left(\theta^* - A^{-1}b\right) = \min_{\theta}\frac{\lambda}{2d}\norm{\theta}_2^2 + \frac{1}{2d}\left(\theta - \theta^*  +A^{-1}b\right)A(\theta - \theta^* + A^{-1}b)
\end{eqnarray*}
From this we can find that the optimal value for $\theta$ is given by:
\begin{eqnarray}\label{app:eq:l2reg_optimaltheta}
    \hat{\theta} = \mathrm{prox}_{A^{-1}\ R}\left(\theta^* - A^{-1}b\right) = (\lambda I + A)^{-1}\left(A\theta^* - b\right)
\end{eqnarray}
We also find that the optimal value in this case is given by:
\begin{eqnarray*}
    -\frac{1}{2d}\left(A\theta^* - b\right)\left(\lambda I + A\right)^{-1}\left(A\theta^* - b\right) + \frac{1}{2d}\theta^*A\theta^* + \frac{1}{d}\theta^{*T}b + \frac{1}{2d}b^TA^{-1}b
\end{eqnarray*}
We can note that this concentrates to:
\begin{eqnarray*}
    -\frac{1}{2d}\theta^{*T}\left(\frac{1}{k}\sum_{l=1}^k\frac{r_l}{\xi_l}\Sigma_l\right)\left(\lambda I + \frac{1}{k}\sum_{l=1}^k\frac{r_l}{\xi_l}\Sigma_l\right)^{-1}\left(\frac{1}{k}\sum_{l=1}^k\frac{r_l}{\xi_l}\Sigma_l\right)\theta^* \nwl
    -\frac{1}{2dk^2}\sum_{l=1}^k\frac{\beta_l^2d}{n}\tr\Sigma_l\left(\lambda I + \frac{1}{k}\sum_{l=1}^k\frac{r_l}{\xi_l}\Sigma_l\right)^{-1} + \frac{1}{2dk}\theta^*\left(\sum_{l=1}^k\frac{r_l}{\xi_l}\Sigma_l\right)\theta^* \nwl
    + \frac{1}{k^2}\sum_{l=1}^k\frac{\beta_l^2}{n}\tr\Sigma_l\left(\frac{1}{k}\sum_{l'=1}^k\frac{r_{l'}}{\xi_{l'}}\Sigma_{l'}\right)^{-1}
\end{eqnarray*}

As such we have as our total optimization problem:
\begin{eqnarray*}
     \min_{\xi, q\succeq 0}\max_{\beta,r\succeq 0} \sum_{l=1}^k \frac{\beta_l q_l}{2k} - \frac{\xi_lr_l}{2k} + \frac{n}{2}\frac{\beta_l}{\beta_l + q_l}(\sigma_{\nu, l}^2 + \xi_l^2)\nwl
     -\frac{1}{2dk^2}\theta^{*T}\left(\sum_{l=1}^k\frac{r_l}{\xi_l}\Sigma_l\right)\left(\lambda I + \frac{1}{k}\sum_{l=1}^k\frac{r_l}{\xi_l}\Sigma_l\right)^{-1}\left(\sum_{l=1}^k\frac{r_l}{\xi_l}\Sigma_l\right)\theta^* \nwl
    -\frac{1}{2dk^2}\sum_{l=1}^k\frac{\beta_l^2d}{n}\tr\Sigma_l\left(\lambda I + \frac{1}{k}\sum_{l=1}^k\frac{r_l}{\xi_l}\Sigma_l\right)^{-1} + \frac{1}{2dk}\theta^{*T}\left(\sum_{l=1}^k\frac{r_l}{\xi_l}\Sigma_l\right)\theta^*
\end{eqnarray*}

\paragraph{Generalization Error}
We note that the generalization error is given by:
\begin{eqnarray*}
    \bbE\frac{1}{2k}\sum_{l=1}^k\norm{y_{k, new} - \frac{1}{\sqrt{d}}g_{k, new}^T\Sigma^{1/2}_k\hat{\theta}}_2^2\nwl
    = \frac{1}{2k}\sum_k\sigma_{\nu, k}^2 -\frac{2}{d}\theta^{*T}\left(\sum_{l'=1}^k\frac{r_{l'}}{\xi_{l'}}\Sigma_{l'}\right)\left(\lambda I + \frac{1}{k}\sum_{l'=1}^k\frac{r_{l'}}{\xi_{l'}}\Sigma_{l'}\right)^{-1}\Sigma_l\theta^* + \frac{1}{d}\theta^{*T}\Sigma_l\theta^*
    \nwl
    + \frac{1}{d}\theta^{*T}\left(\sum_{l'=1}^k\frac{r_{l'}}{\xi_{l'}}\Sigma_{l'}\right)\left(\lambda I + \frac{1}{k}\sum_{l'=1}^k\frac{r_{l'}}{\xi_{l'}}\Sigma_{l'}\right)^{-1}\Sigma_l\left(\lambda I + \frac{1}{k}\sum_{l'=1}^k\frac{r_{l'}}{\xi_{l'}}\Sigma_{l'}\right)^{-1}\left(\sum_{l'=1}^k\frac{r_{l'}}{\xi_{l'}}\Sigma_{l'}\right)\theta^*\nwl
    +\frac{1}{2k^2}\frac{\beta_l^2}{n}\tr\Sigma_l\left(\lambda I + \frac{1}{k}\sum_{l'=1}^k\frac{r_{l'}}{\xi_{l'}}\Sigma_{l'}\right)^{-1}\Sigma_l\left(\lambda I + \frac{1}{k}\sum_{l'=1}^k\frac{r_{l'}}{\xi_{l'}}\Sigma_{l'}\right)^{-1}
\end{eqnarray*}
Where we have made use of the optimal value of $\hat{\theta}$ given in equation~\eqref{app:eq:l2reg_optimaltheta}.

\subsection{Binary Classification for GMMs with arbitrary covariance matrices}

The following trick will be of use for the purposes of the present section:

{\bf The square root trick. }{\it If $x > 0$, then its square root can be "linearized" as follows: $$ \sqrt{x} = \min_{\beta > 0} \frac{1}{2\beta} + \frac{\beta x}{2} $$}

\subsubsection{Proof of Theorem \ref{th: classification master}}
\begin{proof}
Recall the notation of Section \ref{subsec:bin_class}, and also define $M_1, M_2 \in \mathbb{R}^{\frac{n}{2} \times d}$ as 

$$M^T = 
    ( \underbrace{\begin{matrix}
       \mu_1 & ... &
    \mu_1
    \end{matrix}}_{\frac{n}{2}} \underbrace{\begin{matrix}
       \mu_2 & ... & 
    \mu_2 
    \end{matrix}}_{\frac{n}{2}}) = \left(M_1^T |  M_2^T \right)$$
     $$z^T = \begin{pmatrix}
        \mathds{1}^T_{\frac{n}{2}} & -\mathds{1}^T_{\frac{n}{2}} 
    \end{pmatrix}$$ 
    encode the means and the labels corresponding to each class respectively and the classifier if found via solving
\begin{equation*}
    \min_w \calL((A+M)w-z) + \lambda f(w)
\end{equation*}
Introducing $u:= (A+M)w-z$ and inserting a Lagrange multiplier, we have
\begin{equation*}
    \min_{w,u} \max_v v^TAw + v^TMw-v^Tz-v^T u + \calL(u)  + \lambda f(w)
\end{equation*}
Applying Theorem \ref{thm: gcgmt} to the objective above and denoting $v^T = (v_1^T, v_2^T)$ we arrive at the following:
\begin{equation*}
    \min_{w,u} \max_{v_1, v_2} v^TMw-v^Tz-v^T u + \calL(u) + \lambda f(w) + \sum^2_{i=1} \|v_i\|w^T\Sigma^{\frac{1}{2}}_ig_i + v_i^Th_i\|\Sigma^{\frac{1}{2}}_iw\|
\end{equation*}
Performing the optimization over $v_1, v_2$ yields:
\begin{equation*}
    \min_{w,u} \max_{\beta_1, \beta_2 \ge 0}  \calL(u) + \lambda f(w) + \sum^2_{i=1} \beta_i(w^T\Sigma^{\frac{1}{2}}_ig_i + \|h_i\|\Sigma^{\frac{1}{2}}_iw\|+M_i w - z_i - u_i\|)
\end{equation*}
Using the square root trick we obtain:
\begin{align*}
    &\min_{w,u} \max_{\beta_1, \beta_2 \ge 0} \min_{\tau_1, \tau_2 \ge 0} \calL(u) + \lambda f(w) + \sum^2_{i=1} \beta_i(w^T\Sigma^{\frac{1}{2}}_ig_i + \frac{\tau_i}{2} +  \\
    &+ \frac{1}{2\tau_i}(\frac{n}{2}\|\Sigma^{\frac{1}{2}}_iw\|^2+ \frac{n}{2}(\mu_i^T w)^2 + \frac{n}{2} + \|u_i\|^2 - 2 \|\Sigma^{\frac{1}{2}}_iw\| h_i^T u_i - n(1-2c)\mu_i^T w - 2 u_i^T M_i w + 2 u_i^T z_i))
\end{align*}
Regrouping some terms and swapping the min and the max:
\begin{align}
    & \min_{\tau_1, \tau_2 \ge 0} \max_{\beta_1, \beta_2 \ge 0} \sum_{i=1}^2 \beta_i \frac{\tau_i}{2} +  \min_{w,u} \calL(u) + \lambda f(w) + \sum^2_{i=1} \beta_i(w^T\Sigma^{\frac{1}{2}}_ig_i + \nonumber \\
    &+\frac{1}{2\tau_i}(\frac{n}{2}\|\Sigma^{\frac{1}{2}}_iw\|^2+ \|M_i w\|^2 + \|z_i\|^2 + \|u_i\|^2 + 2 \|\Sigma^{\frac{1}{2}}_iw\| h_i^T u_i - 2 z_i^T M_i w - 2 u_i^T M_i w + 2 u_i^T z_i))
    \label{eq: Lf}
\end{align}
By definitions of $M_1$ and $M_2$,
\begin{align*}
    & u_i^T M_i w = (\mu_i^Tw) (\mathds{1}^T u_i) \\
    & \|M_iw\|^2 = \frac{n}{2}(\mu_i^T w)^2
\end{align*}
Consider the inner optimization of \eqref{eq: Lf}:
\begin{align}
     &\min_{w,u} \calL(u) + \lambda f(w) + \sum^2_{i=1} \beta_i(w^T\Sigma^{\frac{1}{2}}_ig_i +
    \frac{1}{2\tau_i}(\frac{n}{2}\|\Sigma^{\frac{1}{2}}_iw\|^2+ \frac{n}{2}(\mu_i^T w)^2 + \frac{n}{2} + \|u_i\|^2 - \nonumber \\ 
    & - 2 \|\Sigma^{\frac{1}{2}}_iw\| h_i^T u_i  -n(1-2c)\mu_i^T w - 2 (\mu_i^Tw) (\mathds{1}^T u_i) + 2 u_i^T z_i)
    \label{eq: inner_opt}
\end{align}
We focus on optimizing over $u$ first:
\begin{align*}
    \min_u \calL(u) + \sum_{i=1}^2 \frac{\beta_i}{2\tau_i}( \|u_i\|^2 + 2 \|\Sigma^{\frac{1}{2}}_iw\| h_i^T u_i - 2(\mu_i^Tw) (\mathds{1}^T u_i) + 2 z_i^T u_i)
\end{align*}
Assuming $\calL$ is separable, we complete the squares and arrive at:
\begin{align*}
    \min_{u} \sum^2_{i=1} \calL(u_i) + \frac{\beta_i}{2\tau_i}( \| (\mu_i^Tw) \mathds{1} -h_i\|\Sigma^{\frac{1}{2}}_iw\| -z_i - u_i\|^2)
\end{align*}
Noting that the expression above is the Moreau Envelope of $\cal L$ by the definition \ref{def: scalMorenv}:
\begin{align*}
    \min_{\tau_1, \tau_2 \ge 0} \max_{\beta_1, \beta_2 \ge 0} \sum_{i=1}^2 \beta_i \frac{\tau_i}{2} + \min_w \lambda f(w) +  \sum^2_{i=1} \beta_iw^T\Sigma^{\frac{1}{2}}_ig_i + \calM_{\frac{\tau_i}{\beta_i} \ \calL} \left((\mu_i^Tw) \mathds{1} -h_i\|\Sigma^{\frac{1}{2}}_iw\| -z_i\right)
\end{align*}
Let $p_i=\Sigma^{\frac{1}{2}}_iw$ and $\gamma_i = \mu_i^T w$, then \eqref{eq: inner_opt} can be rewritten as:
\begin{align*}
    \min_{w, p, \gamma} \max_{\nu, \eta} \lambda f(w) +  \sum^2_{i=1} \beta_ip_i^Tg_i + \calM_{\frac{\tau_i}{\beta_i} \ \calL} \left(\gamma_i \mathds{1} -h_i\|p_i\| -z_i\right) + \nu_i^T(\Sigma^{\frac{1}{2}}_iw-p_i) + \eta_i(\mu_i^T w-\gamma_i)
\end{align*}
Performing the optimization over the direction of $p_i$ and denoting $\theta_i= \|p_i\|$:
\begin{align*}
    \min_{w, \theta_i \ge 0, \gamma} \max_{\nu, \eta} \lambda f(w) +  \sum^2_{i=1} -\theta_i\|\beta_ig_i-\nu_i\| +\calM_{\frac{\tau_i}{\beta_i} \ \calL} \left(\gamma_i \mathds{1} -\theta_i h_i -z_i\right) + \nu_i^T\Sigma^{\frac{1}{2}}_iw + \eta_i(\mu_i^T w-\gamma_i)
\end{align*}
Applying the square root trick again:
\begin{align*}
    & \min_{w, \theta_i \ge 0, \gamma, \zeta_i\ge 0} \max_{\nu, \eta} \lambda f(w) + \sum^2_{i=1} \calM_{\frac{\tau_i}{\beta_i} \ \calL} \left(\gamma_i \mathds{1} -\theta_i h_i -z_i\right)  -   \frac{\theta_i\zeta_i}{2} - \frac{\theta_i}{2\zeta_i}\|\beta_ig_i-\nu_i\|^2 +  
    \\ 
    & + \nu_i^T\Sigma^{\frac{1}{2}}_iw + \eta_i(\mu_i^T w-\gamma_i)
\end{align*}
Optimizing over $\nu_i$ yields:
\begin{align*}
    & \min_{w, \theta_i \ge 0, \gamma, \zeta_i\ge 0} \max_{ \eta} \lambda f(w) +\sum^2_{i=1} \calM_{\frac{\tau_i}{\beta_i} \ \calL} \left(\gamma_i \mathds{1} -\theta_i h_i -z_i\right) -  \frac{\theta_i\zeta_i}{2} - \frac{d\theta_i\beta_i^2}{2\zeta_i} \\ &+ \frac{\theta_i}{2\zeta_i}\|\beta_ig_i-\frac{\zeta_i}{\theta_i}\Sigma^{\frac{1}{2}}_iw\|^2   + \eta_i(\mu_i^T w-\gamma_i)
\end{align*}
Optimize over $w$:
\begin{align*}
    \min_w \lambda f(w) + \sum^2_{i=1}  \frac{\theta_i}{2\zeta_i}\|\beta_ig_i-\frac{\zeta_i}{\theta_i}\Sigma^{\frac{1}{2}}_iw\|^2 + \eta_i\mu_i^T w
\end{align*}
Opening up the norm:
\begin{align*}
    \min_w \lambda f(w) +  \frac{1}{2}w^T\left(\sum_{i=1}^{2} \frac{\zeta_i}{\theta_i}\Sigma_i\right)w - w^T\left(\sum_{i=1}^{2}\beta_i\Sigma_i^{\frac{1}{2}}g_i - \eta_i\mu_i\right) + \frac{d\theta_i\beta_i^2}{2\zeta_i}
\end{align*}
Define $\Sigma = \sum_{i=1}^{2} \frac{\zeta_i}{\theta_i}\Sigma_i$ and vector $x = \sum_{i=1}^{2}\beta_i\Sigma_i^{\frac{1}{2}}g_i - \eta_i\mu_i$, then the objective above can written in terms of the Moreau Envelope as:
\begin{align*}
    \min_w  \lambda \calM_{\lambda \ f(\Sigma^{-\frac{1}{2}})}\left(\Sigma^{-\frac{1}{2}}x\right)  - \frac{x^T\Sigma^{-1}x}{2} +  \frac{d\theta_i\beta_i^2}{2\zeta_i}
\end{align*}
Plugging everything back into \eqref{eq: Lf}, we derive:
\begin{align*}
    &\min_{ \tau_i \geq 0,\theta_i \ge 0, \zeta_i\ge 0,\gamma_i} \max_{ \beta_i \geq 0, \eta_i}  \lambda \calM_{\lambda \ f(\Sigma^{-\frac{1}{2}})}\left(\Sigma^{-\frac{1}{2}}x\right)- \frac{x^T\Sigma^{-1}x}{2}+\sum^2_{i=1} \beta_i \frac{\tau_i}{2} + \\ & + \calM_{\frac{\tau_i}{\beta_i} \ \calL} \left(\gamma_i \mathds{1} -\theta_i h_i -z_i\right)  -  \frac{\theta_i\zeta_i}{2} -\eta_i\gamma_i
\end{align*}
\end{proof}

\subsubsection{Proof of Theorem \ref{thm: classif_l2}}
\begin{proof}
Under the assumptions of Theorem \ref{thm: classif_l2}, the objective \eqref{eq:classification} turns into the following:
\begin{align*}
    \min_w \|(A+M)w-z\|^2 + \lambda \|w\|^2
\end{align*}
Taking the Fenchel dual of the first term we obtain: 
\begin{align*}
    \min_w \max_{v_1, v_2} v_1^T A_1 \Sigma^{\frac{1}{2}}_1 w +  v_2^T A_2 \Sigma^{\frac{1}{2}}_2 w + v_1^T M_1 w + v_2^T M_2 w - v_1^Tz_1 - v_2^T z_2 - \frac{1}{4} \|v_1\|^2 - \frac{1}{4} \|v_2\|^2 + \lambda \|w\|^2
\end{align*}
Applying Theorem \ref{thm: gcgmt} to the objective above: 
\begin{align*}
    \min_w \max_{v_1,v_2} \lambda \|w\|^2 + \sum^2_{i=1} \|v_i\|w^T\Sigma^{\frac{1}{2}}_ig_i + v_i^Th_i\|\Sigma^{\frac{1}{2}}_iw\| + v_i^T M_i w - v_i^T z_i - \frac{1}{4} \|v_i\|^2
\end{align*}
Optimizing over the direction of $v_i$ yields:
\begin{align*}
    \min_w \max_{\beta_1, \beta_2 \ge 0} \lambda \|w\|^2 + \sum^2_{i=1} \beta_i w^T\Sigma^{\frac{1}{2}}_ig_i + \beta_i \|M_i w + h_i\|\Sigma^{\frac{1}{2}}_iw \| -z_i\|  - \frac{\beta_i^2}{4}
\end{align*}
Using the square root trick and definition of $M_i$ and $z_i$:
\begin{align*}
    \min_w \max_{\beta_1, \beta_2 \ge 0} \min_{\tau_1, \tau_2 \ge 0} \lambda \|w\|^2 + \sum^2_{i=1} \beta_i w^T\Sigma^{\frac{1}{2}}_ig_i + \frac{\beta_i \tau_i}{2} + \frac{\beta_i}{2\tau_i}( \|M_i w - z_i\|^2 + \frac{n}{2} w^T \Sigma_i w ) - \frac{\beta_i^2}{4}
\end{align*}
Introducing $\gamma_i$'s:
\begin{align*}
    \min_w \max_{\beta_1, \beta_2 \ge 0, \gamma_1, \gamma_2} \min_{\tau_1, \tau_2 \ge 0} \lambda \|w\|^2 + \sum^2_{i=1} \beta_i w^T\Sigma^{\frac{1}{2}}_ig_i + \frac{\beta_i \tau_i}{2} + \frac{n\beta_i}{4\tau_i}( \gamma_i(\mu_i^Tw - (-1)^{i+1}) -\frac{\gamma_i^2}{4}+  \|\Sigma^{\frac{1}{2}}_iw\|^2) - \frac{\beta_i^2}{4}
\end{align*}
Changing orders of optimization:
\begin{align*}
    \min_{\tau_1, \tau_2 \ge 0} \max_{\beta_1, \beta_2 \ge 0, \gamma_1, \gamma_2} \sum_{i=1}^2 \frac{\beta_i \tau_i}{2} - \frac{n\beta_i \gamma_i^2}{16\tau_i} - \frac{n \beta_i \gamma_i (-1)^{i+1}}{4\tau_i} - \frac{\beta_i^2}{4} +  \\  +\min_w \left[\sum_{i=1}^2 \beta_i \Sigma^{\frac{1}{2}}_ig_i  + \frac{n\beta_i \gamma_i}{4\tau_i} \mu_i\right]^Tw + w^T (\lambda I + \frac{n\beta_1}{4\tau_1} \Sigma_1 + \frac{n\beta_2}{4\tau_2} \Sigma_2) w
\end{align*}
Optimizing over $w$ we are left with the following: 
\begin{align}
   & \min_{\tau_1, \tau_2 \ge 0} \max_{\beta_1, \beta_2 \ge 0, \gamma_1, \gamma_2} -\frac{1}{4} \left[\sum_{i=1}^2 \beta_i \Sigma^{\frac{1}{2}}_ig_i  + \frac{n\beta_i \gamma_i}{4\tau_i} \mu_i\right]^T (\lambda I + \frac{n\beta_1}{4\tau_1} \Sigma_1 + \frac{n\beta_2}{4\tau_2} \Sigma_2)^{-1} \left[\sum_{i=1}^2 \beta_i \Sigma^{\frac{1}{2}}_ig_i  + \frac{n\beta_i \gamma_i}{4\tau_i} \mu_i\right] + \nonumber \\
   & + \sum_{i=1}^2 \frac{\beta_i \tau_i}{2} - \frac{n\beta_i \gamma_i^2}{16\tau_i} - \frac{n \beta_i \gamma_i (-1)^{i+1}}{4\tau_i} - \frac{\beta_i^2}{4}
   \label{eq: obj_l2l2_unsimplified}
\end{align}
Due to our assumptions on $\mu_i$'s we have:
\begin{align*}
    \sum_{i=1}^2 \beta_i \Sigma^{\frac{1}{2}}_ig_i  + \frac{n\beta_i \gamma_i}{4\tau_i} \mu_i \sim \calN(0, \frac{n^2 \beta_1 \beta_2 \gamma_1 \gamma_2}{8 \tau_1 \tau_2}rI+\sum_{i=1}^2 \beta_i^2 \Sigma_i + \frac{n^2 \beta_i^2 \gamma_i^2}{16 \tau_i^2} I) \\
\end{align*}
In other words,
\begin{align*}
    \sum_{i=1}^2 \beta_i \Sigma^{\frac{1}{2}}_ig_i  + \frac{n\beta_i \gamma_i}{4\tau_i} \mu_i \sim  (\frac{n^2 \beta_1 \beta_2 \gamma_1 \gamma_2}{8 \tau_1 \tau_2}rI+\sum_{i=1}^2 \beta_i^2 \Sigma_i + \frac{n^2 \beta_i^2 \gamma_i^2}{16 \tau_i^2} I)^{\frac{1}{2}} G \text{ where }G \sim \calN (0, I)
\end{align*}
Simplify the first term of \eqref{eq: obj_l2l2_unsimplified} via leveraging the corresponding Gaussian concentration:
\begin{align*}
     &\left[\sum_{i=1}^2 \beta_i \Sigma^{\frac{1}{2}}_ig_i  + \frac{n\beta_i \gamma_i}{4\tau_i} \mu_i\right]^T (\lambda I + \frac{n\beta_1}{4\tau_1} \Sigma_1 + \frac{n\beta_2}{4\tau_2} \Sigma_2)^{-1} \left[\sum_{i=1}^2 \beta_i \Sigma^{\frac{1}{2}}_ig_i  + \frac{n\beta_i \gamma_i}{4\tau_i} \mu_i\right]  = \\
     & = \tr  \left ( \bbE \left[\sum_{i=1}^2 \beta_i \Sigma^{\frac{1}{2}}_ig_i  + \frac{n\beta_i \gamma_i}{4\tau_i} \mu_i\right]^T (\lambda I + \frac{n\beta_1}{4\tau_1} \Sigma_1 + \frac{n\beta_2}{4\tau_2} \Sigma_2)^{-1} \left[\sum_{i=1}^2 \beta_i \Sigma^{\frac{1}{2}}_ig_i  + \frac{n\beta_i \gamma_i}{4\tau_i} \mu_i\right]\right) = \\ 
     & = \tr \left( \bbE \left[\sum_{i=1}^2 \beta_i \Sigma^{\frac{1}{2}}_ig_i  + \frac{n\beta_i \gamma_i}{4\tau_i} \mu_i\right]\left[\sum_{i=1}^2 \beta_i \Sigma^{\frac{1}{2}}_ig_i  + \frac{n\beta_i \gamma_i}{4\tau_i} \mu_i\right]^T (\lambda I + \frac{n\beta_1}{4\tau_1} \Sigma_1 + \frac{n\beta_2}{4\tau_2} \Sigma_2)^{-1} \right)  = \\
    & = \tr \left ( (\frac{n^2 \beta_1 \beta_2 \gamma_1 \gamma_2}{8 \tau_1 \tau_2}rI+\sum_{i=1}^2 \beta_i^2 \Sigma_i + \frac{n^2 \beta_i^2 \gamma_i^2}{16 \tau_i^2} I) (\lambda I + \frac{n\beta_1}{4\tau_1} \Sigma_1 + \frac{n\beta_2}{4\tau_2} \Sigma_2)^{-1} \right) = \\
     & = \tr \left((\lambda I + \frac{n\beta_1}{4\tau_1} \Sigma_1 + \frac{n\beta_2}{4\tau_2} \Sigma_2)^{-1} (\frac{n^2 \beta_1 \beta_2 \gamma_1 \gamma_2}{8 \tau_1 \tau_2}rI+\sum_{i=1}^2 \beta_i^2 \Sigma_i + \frac{n^2 \beta_i^2 \gamma_i^2}{16 \tau_i^2} I) \right)
\end{align*}
Plugging everything back into \eqref{eq: obj_l2l2_unsimplified} we deduce: 
\begin{align*}
     &\min_{\tau_1, \tau_2 \ge 0} \max_{\beta_1, \beta_2 \ge 0, \gamma_1, \gamma_2} -\frac{1}{4}  \tr \left((\lambda I + \frac{n\beta_1}{4\tau_1} \Sigma_1 + \frac{n\beta_2}{4\tau_2} \Sigma_2)^{-1} (\frac{n^2 \beta_1 \beta_2 \gamma_1 \gamma_2}{8 \tau_1 \tau_2}rI+\sum_{i=1}^2 \beta_i^2 \Sigma_i + \frac{n^2 \beta_i^2 \gamma_i^2}{16 \tau_i^2} I)\right) + \\
     &+\sum_{i=1}^2 \frac{\beta_i \tau_i}{2} - \frac{n\beta_i \gamma_i^2}{16\tau_i} - \frac{n \beta_i \gamma_i (-1)^{i+1}}{4\tau_i} - \frac{\beta_i^2}{4}
\end{align*}
\end{proof}
\subsubsection{Classification Error}
An attentive reader might have noticed that we held up the classification error derivations to this point. Recall that the classification error for a weight $\hat{w}$ and a general GMM is given by: 
\begin{equation*}
   E(\hat w) = \frac{1}{2}Q\left(\frac{\mu_1^T\hat{w}}{\sqrt{\hat{w}^T \Sigma_1 \hat{w}}}\right) + \frac{1}{2}Q\left(-\frac{\mu_2^T\hat{w}}{ \sqrt{\hat{w}^T \Sigma_2 \hat{w}}}\right) 
\end{equation*}
The proof is a straightforward calculation but we provide it here for completeness:  

\begin{proof} Note that 
    $$E(\hat{w}) = \frac{1}{2}\mathbb{P}_{x \sim \mathcal{N}(\mu_1, \Sigma_1)}(\hat{w}^Tx < 0) + \frac{1}{2}\mathbb{P}_{x \sim \mathcal{N}(\mu_2, \Sigma_2)}(\hat{w}^Tx > 0)$$
    Subtracting the means from the sample point we have: $x = \mu_1 + \eta_1$ for $x \sim \mathcal{N}(\mu_1, \Sigma_1)$ and $x = \mu_2 + \eta_2$ for $x \sim \mathcal{N}(\mu_2, \Sigma_2)$ we obtain:
    $$E(\hat{w}) = \frac{1}{2}\mathbb{P}_{\eta_1 \sim \mathcal{N}(0, \Sigma_1)}(\hat{w}^Ty_1 < -\mu_1^T\hat{w}) + \frac{1}{2}\mathbb{P}_{\eta_2 \sim \mathcal{N}(0, \Sigma_2)}(\hat{w}^T\eta_2 > - \mu_2^T\hat{w})$$
    Noticing that $\hat{w}^T\eta_1 \sim \mathcal{N}(0, w^T\Sigma_1w)$ and $\hat{w}^T\eta_2 \sim \mathcal{N}(0, w^T\Sigma_2w)$ we derive:
\begin{align*}
    & E(\hat{w}) = \frac{1}{2}\mathbb{P}_{\zeta_1 \sim \mathcal{N}(0, \hat{w}^T\Sigma_1\hat{w})}(\zeta_1 < -\mu_1^T\hat{w}) + \frac{1}{2}\mathbb{P}_{\zeta_2 \sim \mathcal{N}(0, \hat{w}^T\Sigma_2\hat{w})}(\zeta_2 > - \mu_2^T\hat{w}) = \\
    & = \frac{1}{2}\mathbb{P}_{\xi_1 \sim \mathcal{N}(0, 1)}(\xi_1 < - \frac{\mu_1^T\hat{w}}{\sqrt{\hat{w}^T \Sigma_1 \hat{w}}}) + \frac{1}{2}\mathbb{P}_{\xi_2 \sim \mathcal{N}(0, 1)}(\xi_2 > -\frac{\mu_2^T\hat{w}}{ \sqrt{\hat{w}^T \Sigma_2 \hat{w}}}) = \\
    & =  \frac{1}{2}Q\left(\frac{\mu_1^T\hat{w}}{\sqrt{\hat{w}^T \Sigma_1 \hat{w}}}\right) + \frac{1}{2}Q\left(-\frac{\mu_2^T\hat{w}}{ \sqrt{\hat{w}^T \Sigma_2 \hat{w}}}\right) 
\end{align*}
\end{proof}
The classification error part for Theorem \ref{th: classification master} follows immediately then by the definitions of $\gamma_i$ and $\Theta_i$ for that case. We now proceed to derive the expressions for the classification error for Theorem \ref{thm: classif_l2} and Corollary \ref{cor: rk_1_noise}. Note that by definition
\begin{align*}
    & \gamma_i = 2 (\mu_i^T w - (-1)^{i+1}) \\
    &  \tau_i^2 = \frac{n}{2} ( \frac{\gamma_i^2}{4} +  \|\Sigma^{\frac{1}{2}}_i w\|^2 ) = \frac{n}{2} ( \frac{\gamma_i^2}{4} +  w^T \Sigma_i w ) 
\end{align*}
Therefore:
\begin{align*}
    \frac{1}{2} Q\left(\frac{\mu_1^{T}w}{\sqrt{w^T\Sigma_1w}}\right) +  \frac{1}{2} Q\left(-\frac{\mu_2^{T}w}{\sqrt{w^T\Sigma_2w}}\right)= \frac{1}{2} Q\left(\frac{\gamma_1+2}{\sqrt{\frac{8}{n} \tau_1^2 - \gamma_1^2}}\right) + \frac{1}{2} Q\left(\frac{-\gamma_2+2}{\sqrt{\frac{8}{n} \tau_2^2 - \gamma_2^2}}\right)
\end{align*}

\subsection{Proof of Corollary \ref{cor: rk_1_noise}}
\begin{proof}
Recall the model: 
\begin{align*}
    & \Sigma_1 = \sigma_1^2I + \nu_1 \nu_1^T  \\
    & \Sigma_2 = \sigma_2^2I + \nu_2 \nu_2^T \\
    & \text{ where } \nu_1, \nu_2 \in \mathcal{N}(0, \sigma^2I)
\end{align*}
Denote the space spanned by $\nu_1$ and $\nu_2$ by $V$ and also denote 
$$T = \lambda I + \frac{n\beta_1}{4\tau_1} \Sigma_1 + \frac{n\beta_2}{4\tau_2} \Sigma_2 $$
Note that both $V$ and $V^{\perp}$ form invariant subspaces of the operator $T$. Also, since we are interested in the asymptotic regime $d \to \infty$, we can take $\|\nu_1\| = \|\nu_2\| \approx \sigma^2d$ and $\nu_1^T \nu_2 \approx 0$ due to the Gaussian concentration. Therefore, we can write:
\begin{align*}
    &T\nu_1 = (\lambda + \frac{n\beta_1}{4\tau_1}(\sigma_1^2 + \|\nu_1\|^2) + \frac{n\beta_2}{4\tau_2}\sigma_2^2 )\nu_1  \\
    & T\nu_2 = (\lambda + \frac{n\beta_1}{4\tau_1} \sigma_1^2+\frac{n\beta_2}{4\tau_2}(\sigma_2^2 + \|\nu_2\|^2)) \nu_2 
\end{align*}
Decomposing $\mathbb{R}^d = V^{\perp} \oplus \nu_1 \oplus \nu_2$ and using that for any operator $F$ one in this case has $\tr F = \tr F_{|V^{\perp}} + \tr F_{|\nu_1} + \tr F_{|\nu_2}$, we derive:
\begin{align*}
     & \tr \left[(\lambda I + \frac{n\beta_1}{4\tau_1} \Sigma_1 + \frac{n\beta_2}{4\tau_2} \Sigma_2)^{-1} (\frac{n^2 \beta_1 \beta_2 \gamma_1 \gamma_2}{8 \tau_1 \tau_2}rI+\sum_{i=1}^2 \beta_i^2 \Sigma_i + \frac{n^2 \beta_i^2 \gamma_i^2}{16 \tau_i^2} I)\right]= \\
     & = \tr \left[(\lambda  + \frac{n\beta_1}{4\tau_1} \sigma^2_1 + \frac{n\beta_2}{4\tau_2} \sigma^2_2 ) ^{-1} (\frac{n^2 \beta_1 \beta_2 \gamma_1 \gamma_2}{8 \tau_1 \tau_2}r+\sum_{i=1}^2 \beta_i^2 \sigma^2_i + \frac{n^2 \beta_i^2 \gamma_i^2}{16 \tau_i^2}) I\right]_{|V_{\perp}} + \\
     & + (\lambda  + \frac{n\beta_1}{4\tau_1} (\sigma^2_1 + \sigma^2d) + \frac{n\beta_2}{4\tau_2} \sigma^2_2 ) ^{-1} (\frac{n^2 \beta_1 \beta_2 \gamma_1 \gamma_2}{8 \tau_1 \tau_2}r+ \beta_1^2 (\sigma^2_1 + \sigma^2d) + \beta_2^2 \sigma^2_2 + \frac{n^2 \beta_i^2 \gamma_i^2}{16 \tau_i^2}) + \\
     & + (\lambda  + \frac{n\beta_1}{4\tau_1} \sigma^2_1  + \frac{n\beta_2}{4\tau_2} (\sigma^2_2 + \sigma^2d)) ^{-1} (\frac{n^2 \beta_1 \beta_2 \gamma_1 \gamma_2}{8 \tau_1 \tau_2}r+ \beta_1^2 \sigma^2_1 + \beta_2^2 (\sigma^2_2+ \sigma^2d ) + \frac{n^2 \beta_i^2 \gamma_i^2}{16 \tau_i^2}) = \\
     & = (d-2)(\lambda  + \frac{n\beta_1}{4\tau_1} \sigma^2_1 + \frac{n\beta_2}{4\tau_2} \sigma^2_2 ) ^{-1} (\frac{n^2 \beta_1 \beta_2 \gamma_1 \gamma_2}{8 \tau_1 \tau_2}r+\sum_{i=1}^2 \beta_i^2 \sigma^2_i + \frac{n^2 \beta_i^2 \gamma_i^2}{16 \tau_i^2}) + \\
     & + (\lambda  + \frac{n\beta_1}{4\tau_1} (\sigma^2_1 + \sigma^2d) + \frac{n\beta_2}{4\tau_2} \sigma^2_2 ) ^{-1} (\frac{n^2 \beta_1 \beta_2 \gamma_1 \gamma_2}{8 \tau_1 \tau_2}r+ \beta_1^2\sigma^2d+ \sum_{i=1}^2 \beta_i^2 \sigma^2_i + \frac{n^2 \beta_i^2 \gamma_i^2}{16 \tau_i^2}) \\
     & + (\lambda  + \frac{n\beta_1}{4\tau_1} \sigma^2_1  + \frac{n\beta_2}{4\tau_2} (\sigma^2_2 + \sigma^2d)) ^{-1} (\frac{n^2 \beta_1 \beta_2 \gamma_1 \gamma_2}{8 \tau_1 \tau_2}r+ \beta_2^2 \sigma^2d + \sum_{i=1}^2 \beta_i^2 \sigma^2_i + \frac{n^2 \beta_i^2 \gamma_i^2}{16 \tau_i^2}) 
\end{align*}
\end{proof}
\begin{remark}
    One might ask where exactly we used the fact that $\nu_1$ and $\nu_2$ are independent from $\mu_1$ and $\mu_2$ in the derivations above. The answer is that we cannot take the expectation over $\mu_1$ and $\mu_2$ first to get the expression with the trace we used if $\nu_1$ and $\nu_2$ depend on $\mu_1$ and $\mu_2$.  
\end{remark}
\newpage
\bibliography{bibliography}

\end{document}

%% file: symbols.tex
\newcommand{\bbeta}{\bm{\beta}}

\newcommand{\btheta}{\bm{\theta}}

\newcommand{\bnu}{\bm{\nu}}

\newcommand{\barphi}{\bar{\phi}}

\newcommand{\bE}{\mathbf{E}}

\newcommand{\be}{\mathbf{e}}

\newcommand{\bg}{\mathbf{g}}

\newcommand{\bq}{\mathbf{q}}
\newcommand{\br}{\mathbf{r}}

\newcommand{\by}{\mathbf{y}}

\newcommand{\bxi}{\bm{\xi}}

\newcommand{\bbE}{\mathbb{E}}
\newcommand{\bbR}{\mathbb{R}}
\newcommand{\bbP}{\mathbb{P}}

\newcommand{\calA}{\mathcal{A}}

\newcommand{\calE}{\mathcal{E}}

\newcommand{\calL}{\mathcal{L}}
\newcommand{\calN}{\mathcal{N}}
\newcommand{\calM}{\mathcal{M}}
\newcommand{\calP}{\mathcal{P}}
\newcommand{\calS}{\mathcal{S}}

\newcommand{\bzero}{\mathbf{0}}

\newcommand{\minl}{\min\limits}

\newcommand{\tr}{\text{Tr}}

\newcommand{\norm}[1]{\left\|#1\right\|}

\newcommand{\nwl}{\nonumber\\}